\documentclass[final,leqno]{siamltex}

\newcommand{\email}[1]{\protect\href{mailto:#1}{#1}}

\title{Optimal Learning for Stochastic Optimization with Nonlinear Parametric Belief Models}
\author{
  Xinyu He\thanks{Department of Electrical Engineering, Princeton University, Princeton, NJ, 08544 (\email{xinyuhe@princeton.edu}).}
  \and
  Warren B. Powell\thanks{Department of Operations Research and Financial Engineering, Princeton University, Princeton, NJ, 08544 (\email{powell@princeton.edu}).}
}

\usepackage{xcolor}
\colorlet{siaminlinkcolor}{green!50!black}
\colorlet{siamexlinkcolor}{red!50!black}

\usepackage[colorlinks=true, pdfstartview=FitV, linkcolor=siaminlinkcolor, citecolor=siaminlinkcolor, urlcolor=siamexlinkcolor]{hyperref}
\usepackage{url}

\usepackage{booktabs,comment,multicol,multirow,float,makecell,graphicx,amsmath,url,threeparttable}
\usepackage{bm}
\usepackage{amssymb, amsfonts, dsfont}
\usepackage{comment}
\usepackage{setspace}
\usepackage{caption,subcaption}

\usepackage{nccmath}

\usepackage{extarrows}

\usepackage{algorithm}
\usepackage{algorithmic}

\newcommand{\SWITCH}[1]{\STATE \textbf{switch} (#1)}
\newcommand{\ENDSWITCH}{\STATE \textbf{end switch}}
\newcommand{\CASE}[1]{\STATE \textbf{case} #1\textbf{:} \begin{ALC@g}}
\newcommand{\ENDCASE}{\end{ALC@g}}

\newcommand{\DEFAULT}{\STATE \textbf{default:} \begin{ALC@g}}
\newcommand{\ENDDEFAULT}{\end{ALC@g}}
\newcommand{\DEFAULTLINE}[1]{\STATE \textbf{default:} }

\usepackage{chngcntr}
\counterwithin{figure}{section}

\newcommand\numberthis{\addtocounter{equation}{1}\tag{\theequation}}

\def\d{\delta}
\def\e{{\epsilon}}
\def\k{\theta}
\begin{document}

\maketitle
\begin{abstract}
We consider the problem of estimating the expected value of information (the knowledge gradient) for Bayesian learning problems where the belief model is nonlinear in the parameters.  Our goal is to maximize some metric, while simultaneously learning the unknown parameters of the nonlinear belief model, by guiding a sequential experimentation process which is expensive.  We overcome the problem of computing the expected value of an experiment, which is computationally intractable, by using a sampled approximation, which helps to guide experiments but does not provide an accurate estimate of the unknown parameters.  We then introduce a resampling process which allows the sampled model to adapt to new information, exploiting past experiments.  We show theoretically that the method converges asymptotically to the true parameters, while simultaneously maximizing our metric.  We show empirically that the process exhibits rapid convergence, yielding good results with a very small number of experiments.
\end{abstract}


\section{Introduction}

We consider the following problem: we have a function over a finite number of alternatives, where the expression of the function is known except for some unknown parameters. Our goal is both to learn the correct parameters and to find the alternative that maximizes the function. Information can be collected by measuring the function value at any chosen alternative. However, the measurements are noisy and expensive, and we only have a limited budget to evaluate these alternatives. After exhausting the budget, we have to provide the best estimate of the unknown parameters and identify the best alternative to maximize our metric.


The optimization of functions with noisy measurements has been broadly studied in the stochastic search community. \cite{Spall2003} provides an extensive review of different methods for stochastic search. Some commonly used algorithms include gradient estimation \cite{Fu2006}, response surface methodology \cite{Barton2006}, and metaheuristics such as tabu search and genetic algorithms (\cite{Olafsson2006}, also see \cite{Bianchi2009} for more on metaheuristics). These methods are usually designed for functions that are relatively easy to evaluate, while our focus is on problems where a function evaluation is expensive (and we do not have access to derivatives).

One of the main challenges of the problems with expensive experiments is known as ``\textit{exploration vs exploitation}'' (see Chapter 12 of \cite{Powell2011}), which requires a balance of exploiting the current optimal solutions and exploring with uncertain outcomes to learn the problem. Specifically, for a finite number of alternatives, some heuristics to strike this balance include epsilon-greedy (\cite{Sutton1998}, also see \cite{Singh2000} for convergence analysis), interval estimation \cite{Kaelbling1993}, and Chernoff interval estimation \cite{Streeter2006}. These methods usually perform well on small problems but have scaling issues in higher dimensions. More sophisticated methods to collect information efficiently include Gittins indices \cite{Gittins2011}, upper confidence bound methods \cite{Auer2002},  and expected improvement (\cite{Jones1998}, also see \cite{Gramacy2011} for recent work). 


This problem is an extension of the ranking and selection (R\&S) problem (see \cite{Kim2006,Hong2009} and the references cited there), which focuses on finding the optimal alternative, as opposed to learning the values of an underlying, parametric belief model. There have been two major approaches to R\&S problems: the frequentist approach and the Bayesian approach. The frequentist approach assumes that information only comes from observations and uses the distribution of the observed data to estimate the parameters (see, e.g., \cite{Hastie2009,Kim2006,Audibert2010}). The Bayesian approach begins with a probabilistic belief model about the true parameters which are updated as new observations are collected (see, e.g., \cite{Chen2000,Chick2006}). For the Bayesian approach, there are mainly two directions: the Optimal Computing Budget Allocation (OCBA) (see, e.g., \cite{Chen1995,Chen2000,He2007}), which tries to maximize the posterior probability of correct selection, and the Value of Information (VoI) procedures (see, e.g., \cite{Chick2001a,Frazier2008}), which maximizes the improvement in a single experiment. Below we give a brief review of the VoI procedures, with special attention given to the knowledge gradient \cite{Frazier2008}.


Since every experiment is expensive or time consuming (consider for example, the laboratory experiments in materials science that might take days or weeks), it is important to maximize the value of each measurement. \cite{Gupta1996} proposes the idea of computing the marginal value of information from a single experiment. Based on this, \cite{Frazier2008} extends the idea using the Bayesian approach, and presents the Knowledge Gradient (KG) policy. It analyzes the case where all alternatives are independent, and also proves that KG is the only stationary policy that is both myopically and asymptotically optimal. \cite{Frazier2009} adapts the knowledge gradient to handle correlations in the beliefs about discrete alternatives. Both \cite{Frazier2008} and \cite{Frazier2009} use lookup table belief models, and become computationally expensive when the number of alternatives becomes large, as typically happens when an experimental choice involves multiple dimensions. \cite{Negoescu2011} is the first paper that studied KG using a parametric belief model. It assumes that the belief model is linear in some unknown parameters, and imposes the uncertainty of the function values onto the parameters. This strategy reduces the number of parameters to be estimated from the number of alternatives (required by a lookup table belief model) to a much lower-dimensional parameter vector. While the knowledge gradient is easy to compute for belief models that are linear in the parameters, nonlinear belief models are computationally intractable. \cite{Chen2014} studies the more general nonlinear parametric model, and proposes the Knowledge Gradient with Discrete Priors (KGDP), which overcomes the computational issues, but requires that we represent the uncertainty about the unknown parameters using a finite number of samples, where one of the samples has to be the true parameter. KGDP is able to handle any nonlinear belief model, but the assumption that one of the candidates is correct (we refer to this assumption later as the \textit{truth from prior} assumption) seems too strong in most real world settings, especially when the parameters have four or more dimensions. Moreover, \cite{Chen2014} also fails to give any theoretical proof of convergence of KGDP given the truth from prior assumption.

In this paper, we present a resampling algorithm that works with KGDP to find not only the best alternative but also the correct parameters without the assumption of one candidate being the truth. We start with several potential candidates, and use the information from the measurements to guide resampling and discover more probable candidates. Regarding the objective of finding the correct parameter, we propose a new metric for calculating the knowledge gradient that minimizes the entropy of the candidate parameters. We also prove the convergence of both the non-resampling (with truth from prior) and the resampling KGDP algorithms. In the experimental section, we apply our algorithm to a real problem in materials science: optimizing the kinetic stability of an experimental problem involving the control of a water-oil-water (W/O/W) nano emulsion \cite{Chen2014}. Compared with \cite{Chen2014}, our algorithm shows significant improvements.

In order to find the correct parameter, we need to deal with the parameter space (denoted as $\k$ space) besides the alternative space (denoted as $x$ space). Here arises a dual problem: in $x$ space, we solve a maximization problem to figure out which alternative maximizes some performance metric (strength, conductivity, reflexivity); in $\k$ space, we solve another optimization problem to locate the most promising candidates. The second optimization problem is solved via minimizing the mean square error (MSE).


To the best of our knowledge, this is the first optimal learning paper that addresses the dual problems of objective maximization with parameter identification for problems with parametric belief models (other literatures with similar dual-objective formulations usually assume that experiments are inexpensive, e.g, \cite{2014arXiv1402}). Previous papers only concentrate on discovering the optimal alternative, but in many real world situations, scientists also care about getting accurate estimates of the parameters. For example, materials scientists may want to know both the tunable variables of an experiment (such as the temperature, the pressure, and the density), and some unknown and uncontrollable parameters, which can help them understand the intrinsic nature of the physical phenomena. Our work in this paper provides a powerful solution to this issue.

This paper makes the following major contributions:
\begin{itemize}
    \item We present a resampling algorithm in the parameter space, which works with the optimal learning methods in the alternative space, to discover both the optimal alternative and the correct parameters.
    \item We propose a new metric to calculate the knowledge gradient, which focuses on reducing the uncertainty of parameters by maximizing entropic loss.
    \item We prove the asymptotic optimality of both the nonresampling (with truth from prior) and the resampling algorithms using either traditional function oriented metric or our new entropy oriented metric.
    \item We show experimentally that our resampling algorithm has impressive performance in both finding the optimal alternative and estimating the parameters.
\end{itemize}

The paper is organized as follows. Section~\ref{sec2} reviews the principles of optimal learning based on maximizing the value of information, along with the knowledge gradient with discrete priors (KGDP) proposed in \cite{Chen2014}. In Section~\ref{sec3}, we introduce the resampling algorithm as well as the entropy oriented metric for calculating the knowledge gradient. Section~\ref{sec5} provides the asymptotic convergence proof of KGDP without resampling under truth from prior assumption, and Section~\ref{sec6} proves asymptotic convergence of KGDP with resampling. We show empirical results on both a simulated and a real-world problem in Section~\ref{sec7}, and finally conclude in Section~\ref{sec8}. 



\section{Knowledge Gradient for Nonlinear Belief Models}\label{sec2}

In this section, we first review the ranking and selection problem along with the knowledge gradient policy, and then introduce the model of Knowledge Gradient with Discrete Priors.


\subsection{Ranking and Selection (R\&S) Problem}

Suppose we have a finite number of alternatives $\mathcal{X}=\{x_1,x_2,...,x_M\}$. Each alternative $x\in \mathcal{X}$ is associated with a true utility $\mu_{x}$, which is presumed unknown to us. The goal is to determine the $x$ that has the largest utility through a budget of $N$ sequential measurements. At time $n$ (i.e, after $n$ measurements), suppose we decide to query alternative $x^n$ according to some policy, and our measurement is $\hat{y}^{n+1}$, where the superscripts imply that $\hat{y}^{n+1}$ will be unknown until the $(n+1)$-{th} measurement. We assume the inherent noise ($W^n$) in each measurement follows a normal distribution with zero mean and variance $\sigma^2$, where $\sigma$ is known to us. That is, for $n=0,1,...,N-1$, 
\small\begin{align*}
\hat{y}^{n+1}=\mu_{x^n} + W^{n+1},
\end{align*}\normalsize
where $W^{n+1}\sim \mathcal{N}(0,\sigma^2)$.


For each $x\in \mathcal{X}$, we  use $\theta^n_x$ as our estimate of the true utility $\mu_x$ after $n$ experiments. Our goal is to select the $x$ with the highest estimate after $N$ measurements, i.e.,
\small\begin{align*}
x^N = \mathop{\text{argmax}}_{x\in\mathcal{X}}\theta_x^N.
\end{align*}\normalsize


\subsection{Knowledge Gradient}

Let $S^n$ denote the state of knowledge at time $n$, and let $V^n(S^n)$ be the value of information if we are in state $S^n$. In the R\&S problem, we have $V^n(S^n)=\max_{x'\in\mathcal{X}}\theta^n_{x'}$. The transition from $S^n$ to $S^{n+1}$ occurs when we take a measurement at $x^{n}=x$ and observe $\hat{y}^{n+1}$, with
\small\begin{align*}
V^{n+1}(S^{n+1}(x))= \max_{x'\in\mathcal{X}}\theta_{x'}^{n+1}(x).
\end{align*}\normalsize

At time $n$, $\theta_{x'}^{n+1}(x)$ is a random variable since it depends on the noise $W^{n+1}$. We would like to maximize our expected incremental information from the next measurement, which we call the \textit{knowledge gradient}. At time $n$, the knowledge gradient of $x$ is defined as:
\small\begin{align}\label{eq_KG}
\nu^{KG,n}(x) &= \mathbb{E}^n\left[V^{n+1}(S^{n+1}(x))-V^n(S^n)\right] \notag\\
&= \mathbb{E}^n \left[\max_{x'\in\mathcal{X}}\theta_{x'}^{n+1}(x)|S^n=s,x^n=x\right] - \max_{x'\in\mathcal{X}}\theta^n_{x'}.
\end{align}\normalsize

In each iteration, the \textit{Knowledge Gradient} policy measures the alternative with the largest knowledge gradient:
\small\begin{align*}
x^{KG,n} = \mathop{\text{argmax}}_{x\in\mathcal{X}}\nu^{KG,n}(x).
\end{align*}\normalsize

This is a myopic policy that maximizes the value of information from a single experiment \cite{Frazier2008}.

The knowledge gradient Equation~(\ref{eq_KG}) is easy to calculate when using a lookup table belief model (\cite{Frazier2008} and \cite{Frazier2009}), or when the belief model is linear in the parameters \cite{Negoescu2011}, with the form ${f}(x;\k) = \k_0 + \k_1 \phi_1(x) + ... +\k_n \phi_n(x)$. We run into problems, however, if the function is nonlinear in the parameter vector $\k$, since we have to incorporate the updating mechanism in the representation of $\theta^{n+1}_{x'}(x)$ in Equation~(\ref{eq_KG}).  When this is imbedded within both the maximum and expectation operator, the resulting expression becomes computationally intractable.  We address this issue in the next section.


\subsection{Knowledge Gradient with Discrete Priors (KGDP)}

The Knowledge Gradient with Discrete Priors (KGDP) \cite{Chen2014} is designed to handle parametric belief models that are nonlinear in the parameters. Suppose we have a function $f(x;\k)$, where $x$ is the alternative and $\k$ represents the unknown parameters. The expression of $f(x;\k)$ is known to us except for the unknown $\k$. Let $\k^*$ denote the true parameter. Our goal is to find the $x$ that maximizes the true function $f(x;\k^*)$. KGDP assumes that $f(x;\k^*)$ can be approximated as a convex combination of $L$ candidates:
\small\begin{align*}
f(x;\k^*) \approx \bar{f}^n(x)=\sum_{i=1}^Lf(x;\k_i)p_i^n,
\end{align*}\normalsize
where $\k_i$'s are known as the candidates, with $p_i^n$'s as their probabilities at time $n$.

KGDP requires that one of the candidates is equal to or very close to the truth. We call this the \textit{truth from prior }assumption, while its opposite, {\it truth not from prior}, recognizes that the true $\k$ may not belong to our sampled set. As we take measurements, the $\k$'s are fixed, while the probabilities $p_i$'s get updated after each iteration, which means that the belief is naturally conjugate. The state variable $S^n$ is defined as the probability vector, $S^n=(p_1^n,p_2^n,...,p_L^n,)$. The KGDP score of each alternative $x$ can be calculated according to Equation~(\ref{eq_KG}) (we review this formula in \hyperref[sec_KGDP_eq]{Section \ref{sec_KGDP_eq}}). We calculate the KGDP scores of all alternatives, and select the one with the highest score to evaluate next. After we exhaust the budget, the true parameters can be estimated by choosing the most probable candidate, i.e, $\hat{\k}^* = \k_i$, where $i=\mathop{\text{argmax}}_{l\in[L]} p_l^N$.
KGDP provides a powerful approach to handling nonlinear parametric models if we use a sampled belief model for the unknown parameters. However, the complexity of KGDP grows quickly as the number of candidates increases. In the experiments described in \cite{Chen2014}, it can handle at most tens of candidates. Moreover, the truth from prior assumption may be a reasonable approximation if $\k$ has one or two dimensions, but in higher dimensions it is unlikely that this would be the case.  This is particularly problematic when we are interested in finding the best possible estimate of the parameters themselves.  For this reason, we propose a resampling algorithm that can adaptively find more promising parameter candidates.


\section{Optimal Learning for Parameters and Alternatives}\label{sec3}


Our dual-objective optimization problem requires handling the alternative space and the parameter space simultaneously. Note that while finding the best alternative and learning the best estimate of the parameter vector each have their own metrics, they are not conflicting goals, since the best estimate of the parameters can lead to the best design.

Traditionally, the optimal learning literature has focused on optimizing some metric such as choosing the best parameters to maximize the strength of a material or minimize the number of defects.  When we use parametric belief models, we can achieve these goals by finding good estimates of the unknown parameters, and then optimizing a (deterministic) nonlinear model based on the estimates.  To capture the value of learning the unknown parameters correctly, we can replace our original metric with the entropy of the belief vector.


\subsection{The Resampling Scheme}

As in KGDP, we still keep a set of $L$ candidates, but they change over time. Let $\mathbb{L}^n=\{\k_1^n,...,\k_L^n\}$ denote the candidate set at time $n$, and $\vec{p^n}=\{p_1^n,...,p_L^n\}$ denote the probabilities. Assume that the true function $f(x;\k^*)$ can be approximated as
\small\begin{align*}
f(x;\k^*)\approx\bar{f}^n(x)=\sum_{i=1}^Lf(x;\k_i^n)p_i^n.
\end{align*}\normalsize


The candidates are no longer fixed and resampling happens periodically, in which the least probable candidates get replaced by more probable ones according to the whole history of experiments. Hence, we need to search two spaces: in the outer loop, we work in the $x$ space to collect as much information as possible; in the inner loop, we search the $\k$ space for more promising $\k$'s. 


We use the following process for testing our resampling strategy.  We start by choosing a large sample of $K$ (say $K=1,000,000$) realizations of $\k$'s, and then choose one of these as the truth. We call the $K$ samples the \textit{large pool}, denoted as $\mathbb{K}$. Then, we choose a small sample set $\mathbb{L}$ as our sampled prior.  The idea is that with such a large base population, it will be very unlikely that one of the $\k$'s in our sampled prior will be the truth (or even close to it).

As the candidates' probabilities change over time, under certain conditions resampling is triggered. We use mean square error (MSE) as the criterion to choose $\k$'s in the resampling steps.




\subsection{Mean Square Error (MSE) Formulation}

We first introduce the formulation of the MSE problem, which is equivalent to the maximum likelihood estimation (MLE) in Gaussion noise settings. Suppose that after $n$ measurements, the history of experimental results is $\mathcal{F}^n=\sigma\{(x^0,\hat{y}^1),...,(x^{n-1},\hat{y}^n)\}$. The likelihood of each $\k_i$ is given by:
\small\begin{align*}
\mathcal{L}(\k_i|\mathcal{F}^n)=\prod_{j=1}^n\exp\left(-\frac{[\hat{y}^{j}-f(x^{j-1};\k_i)]^2}{2\sigma^2}\right).
\end{align*}\normalsize

By taking the logarithm of $\mathcal{L}(\k_i|\mathcal{F}^n)$, we have the MSE formula for $\k_i$ as
\small\begin{align*}
MSE(\k_i|\mathcal{F}^n)=\frac{1}{n}\sum_{j=1}^n[\hat{y}^j - f(x^{j-1};\k_i)]^2.
\end{align*}\normalsize

The idea is that each time resampling is triggered, we calculate the MSE of all $\k$'s in the large pool $\mathbb{K}$. Although $K$ can be arbitrarily large, this calculation is considered much faster and more efficient compared to the expensive physical experiments. Then resampling is conducted in the sub-level set of the MSE function $MSE(\k)$. We define a threshold $\d$, and all $\k$'s with MSE smaller than $\d$ form the sub-level set $\mathbb{S}^n$:
\begin{align*}
\mathbb{S}^n=\{\k\in\mathbb{K}|MSE(\k|\mathcal{F}^n)\leq \d\}.
\end{align*}

Alternatively, we can define a \textit{small pool} of a certain size $R$ that contains the $R$ $\k$'s with the smallest MSE, namely
\begin{align*}
\mathbb{S}^n=\{\k\in\mathbb{K}|MSE(\k|\mathcal{F}^n)\leq MSE_R\}.
\end{align*}
where $MSE_R$ denotes the value of the $R$-th smallest $MSE$. In other words, the small pool $\mathbb{S}^n$ contains the most probable $\k$'s up to now. We resample from the small pool a number of new candidates as needed.

On the one hand, this method rules out the majority of unlikely samples and avoids exploring too large a space inefficiently; on the other hand, it supplies a small pool containing sufficiently promising samples to resample from, thus providing plenty of exploration compared with simply selecting the $L$ samples with the smallest MSE.


\subsection{Updating the Probabilities of Candidates}


For candidate $\k_i^n$, suppose its prior probability is $p_i^{n}$. After the $(n+1)$-th measurement given by $\hat{y}^{n+1}\sim \mathcal{N}(f(x^n;\k_i^n), \sigma^2)$, the likelihood of $\hat{y}^{n+1}$ is given by
\begin{align*}
g_Y(\hat{y}^{n+1}|\k^*=\k_i^n) = \frac{1}{\sqrt{2\pi}\sigma}\exp\left(-\frac{[\hat{y}^{n+1}-f(x^n;\k_i^n)]^2}{2\sigma^2}\right).
\end{align*}

By Bayes' rule, the posterior probability is proportional to the prior probability times the likelihood:
\begin{align}\label{eq_mse}
p^{n+1}_i\propto g_Y(\hat{y}^{n+1}|\k^*=\k_i^n) p^n_i = \frac{1}{\sqrt{2\pi}\sigma}\exp\left(-\frac{[\hat{y}^{n+1}-f(x^n;\k_i^n)]^2}{2\sigma^2}\right)p_i^n.
\end{align}

After normalization, the updating rule of $p_i$ is given by
\begin{align}\label{eq_p_update1}
p_i^{n+1}=\frac{\exp\left(-\frac{[\hat{y}^{n+1}-f(x^n;\k_i^n)]^2}{2\sigma^2}\right)p_i^n}{\sum_{l=1}^L\exp\left(-\frac{[\hat{y}^{n+1}-f(x^n;\k_l^n)]^2}{2\sigma^2}\right)p_l^n}.
\end{align}

After the $(n+1)$-th measurement, if resampling is not triggered, $p_i^{n+1}$ can be updated as Equation~(\ref{eq_p_update1}) for $i=1,2,...,L$. Otherwise, we should calculate the likelihood of $\k_i$ by all the previous measurements, given by
\begin{align}\label{eq_likeli}
g_Y(\hat{y}^{1},...,\hat{y}^{n+1}|\k^*=\k_i^{n+1})=\prod_{j=0}^{n}\exp\left(-\frac{[\hat{y}^{j+1}-f(x^j;\k_i^{n+1})]^2}{2\sigma^2}\right).
\end{align}

Since initially all parameters are associated with equal probability, if resampling happens at time $(n+1)$, $p^{n+1}_i$ is given by
\begin{align}\label{eq_p_update2}
p^{n+1}_i=\frac{\prod_{j=0}^{n}\exp\left(-\frac{[\hat{y}^{j+1}-f(x^j;\k_i^{n+1})]^2}{2\sigma^2}\right)}{\sum_{l=1}^L\prod_{j=0}^{n}\exp\left(-\frac{[\hat{y}^{j+1}-f(x^j;\k_l^{n+1})]^2}{2\sigma^2}\right)},
\end{align}
where $i$ and $l$ index the set of candidates after resampling.

\subsection{Resampling Procedure}

Resampling is triggered under either of two conditions: 1) the same set of candidates have been used for $n^{resamp}$ iterations; 2) more than $L/2$ candidates have probabilities lower than $\e$. The resampling process goes as follows.

1. When resampling is triggered, we first remove the candidates with low probabilities. Denote the set as $\mathbb{L}^n_{rm}$:
\begin{align*}
\mathbb{L}^n_{rm}=\{\k\in\mathbb{L}^n|p^n(\k)\leq\e\}.
\end{align*}
Note that if $p^n(\k)>\e$ for all $\k\in\mathbb{L}^n$, we still select a small portion (say $1$ or $2$) of the least probable ones as $\mathbb{L}^n_{rm}$. In this way, we can avoid getting stuck in a set of unlikely candidates (for example, in an extreme case where we have $L$ identical but wrong candidates, we get stuck if not dropping anyone). 

2. Then we calculate the MSE of each $\k$ in the large pool by Equation~(\ref{eq_mse}), and select $R$ ones with the smallest MSE to form the small pool.

3. Next, calculate the likelihood of each $\k$ in the small pool given by Equation~(\ref{eq_likeli}). We then use the likelihoods as weights and do weighted sampling without replacement to select $|\mathbb{L}^n_{rm}|$ $\k$'s and add them to the candidate set.

4. Once the candidates are updated, we update their probabilities accordingly using the whole measurement history according to Equation~(\ref{eq_p_update2}).

5. Finally, we check if the current set can still trigger the resampling conditions, since the new set may still contain over $L/2$ candidates with probabilities lower than $\e$. If not, resampling finishes. Otherwise, repeat the resampling process.

A detailed flowchart is given in \hyperref[detail_workflow]{Section \ref{detail_workflow}}.

\subsection{Evaluation Metrics}

We introduce two Knowledge Gradient related policies for choosing the next alternative to measure. Corresponding to our dual objectives, they focus on maximizing the performance metric and learning the parameter respectively. The first one, initially given in \cite{Chen2014}, is function value oriented and hence denoted as KGDP-$f$. The second one, denoted as KGDP-$H$, focuses on minimizing the entropy of the belief vector $\vec{p}^n=\{p_1^n,...,p_L^n\}$, which leads to a better estimate of $\k^*$, from which we can then optimize the original function $f(x;\k)$.









\subsubsection{KGDP-$f$}\label{sec_KGDP_eq}

We can measure the expected incremental function value as in \cite{Chen2014}, where the formula of the KGDP-$f$ was originally given. At time $n$, we define $S^n$ as the probability vector $(p_1^n,...,p_L^n)$, $V^n(S^n)$ as the current largest estimate, i.e, $\max_{x\in\mathcal{X}}\bar{f}^n(x)$ (recall that $\bar{f}^n(x)$ is our estimate of $f(x;\k^*)$ at time $n$, given by $\bar{f}^n(x)=\sum_{i=1}^L f(x;\k_i^n)p_i^n$). Let $p^{n+1}(x)$ represent the posterior probability after measuring $x$.  KGDP-$f$ is calculated in \cite{Chen2014} as:
\small
\begin{align}\label{eq_KGDP1}
\nu^{KGDP-f,n}(x) 
= & \mathbb{E}^n\left[\max_{x'}\sum_{i=1}^Lf_i(x')p_i^{n+1}(x)|S^n=s,x^n=x\right]-\max_{x'}\sum_{i=1}^Lf_i(x')p_i^n \notag\\
=& \sum_{j=1}^L\left[\int_{\omega}\max_{x'}\frac1{c_j}\sum_{i=1}^L f_i(x')\exp\left(-\frac{[f_j(x)-f_i(x)+\omega]^2}{2\sigma^2}\right)p_i^ng(\omega)d\omega \right] p_j^n \notag \\
    &-\max_{x'}\sum_{i=1}^L f_i(x')p_i^n,
\end{align}\normalsize
where $i$ and $j$ index the candidates in $\mathbb{L}^n$, $c_j=\sum_{i=1}^L\exp\left[-\frac{(f_j(x)-f_i(x)+\omega)^2}{2\sigma^2}\right]p_i^n$, $g(\omega)=\frac{1}{\sqrt{2\pi}\sigma}\exp\left(-\frac{\omega^2}{2\sigma^2}\right)$, and $f_i(x)$ is short for $f(x;\k_i^n)$.

In each integral, let $f_j(x)+\omega=\hat{y}$, and the above equation can be simplified as
\small
\begin{align}\label{eq_KGDP2}
\nu^{KGDP-f,n}(x)
=&\frac{1}{\sqrt{2\pi}\sigma}\int_{-\infty}^{+\infty}\max_{x'}\left[\sum_{i=1}^L f_i(x')p_i^n\exp\left(-\frac{[\hat{y} - f_i(x)]^2}{2\sigma^2}\right)\right]d\hat{y} \notag\\
&-\max_{x'}\sum_{i=1}^L f_i(x')p_i^n,
\end{align}\normalsize
where $i$ indexes the candidates at time $n$. Compared with Equation~(\ref{eq_KGDP1}) in \cite{Chen2014}, Equation~(\ref{eq_KGDP2}) is simpler for both calculation and theoretical proof.

In KGDP-$f$, our decision at time $n$ is
\begin{align*}
x^n= \mathop{\text{argmax}}_{x\in\mathcal{X}}\nu^{KGDP-f,n}(x).
\end{align*}



\subsubsection{KGDP-$H$}

Entropy is a metric to measure the uncertainty of unknown data, which is widely used especially in information theory. The entropy of the candidates at time $n$ is given by
\small
\begin{align*}
H(p_1^n,...,p_L^n) = -\sum_{i=1}^L p_i^n\log p_i^n.
\end{align*}\normalsize

In KGDP-$H$, we measure the alternative that has the largest expected entropic loss. Define the state variable $S^n$ as the probability vector $(p_1^n,p_2^n,...,p_L^n,)$, and let $V^n(S^n)$ be the entropy. The KGDP-$H$ score is given by
\small
\begin{align}\label{eq_KGDP-H1}
\nu^{KGDP-H,n}(x) &= \mathbb{E}^n\left[\sum_{i=1}^L p_i^{n+1}(x)\log p_i^{n+1}(x) |S^n=s,x^n=x\right]-\sum_{i=1}^L p_i^n\log p_i^n \notag \\
&= \sum_{j=1}^L \left(\int_\omega \sum_{i=1}^L p_{i|j}^{n+1}(x,\omega)\log p_{i|j}^{n+1}(x,\omega) g(\omega) d\omega \right) p_j^n - \sum_{i=1}^L p_i^n\log p_i^n,
\end{align}\normalsize
where $g(\omega)=\frac{1}{\sqrt{2\pi}\sigma}\exp\left(-\frac{\omega^2}{2\sigma^2}\right)$, and $p_{i|j}^{n+1}(x,\omega)$ is the probability of $\k_i$ at time $(n +1)$ given that $\k_j$ is the truth and the noise is $\omega$, given by
\begin{align*}
p_{i|j}^{n+1}(x,\omega)=\frac{\exp\left[-\frac{(f_j(x)-f_i(x)+\omega)^2}{2\sigma^2}\right]p_i^n}{\sum_{k=1}^L \exp\left[-\frac{(f_j(x)-f_k(x)+\omega)^2}{2\sigma^2}\right]p_k^n}.
\end{align*}\normalsize

Alternatively, it can be written as
\small
\begin{align}\label{eq_KGDP-H2}
\nu^{KGDP-H,n}(x)
 =& \int_{-\infty}^{+\infty} \sum_{i=1}^L p_i^{n+1}(x)\log p_i^{n+1}(x)\cdot \frac{1}{\sqrt{2\pi}\sigma}\left(\sum_{i=1}^L p_i^n \exp\left[-\frac{(\hat{y}-f_i(x))^2}{2\sigma^2}\right]\right)d\hat{y} \notag\\
 &\hspace{1.5em}- \sum_{i=1}^L p_i^n\log p_i^n \notag \\
=& \frac{1}{\sqrt{2\pi}\sigma}\int_{-\infty}^{+\infty} \sum_{i=1}^L p_i^n\exp\left[-\frac{(\hat{y}-f_i(x))^2}{2\sigma^2}\right]\log p_i^{n+1}(x)d\hat{y} - \sum_{i=1}^L p_i^n\log p_i^n,
\end{align}\normalsize
where $p_i^{n+1}(x)$ is given by Equation~(\ref{eq_p_update1}).

In KGDP-$H$, our decision at time $n$ is
\small\begin{align*}
x^n= \mathop{\text{argmax}}_{x\in\mathcal{X}}\nu^{KGDP-H,n}(x).
\end{align*}\normalsize




\bigskip
After we exhaust the budget of $N$ experiment, we give our estimates of the optimal alternative $\hat{x}^*$ and parameters $\hat{\k}^*$ by:
\vspace{-1em}
\small
\begin{align*}
\hat{x}^* = \mathop{\text{argmax}}_{x\in\mathcal{X}}\bar{f}^N(x) = \mathop{\text{argmax}}_{x\in\mathcal{X}}\sum_{i=1}^L p_i^N f(x;\k^N_i), 
\hspace{1.5em}
\hat{\k}^* = \mathop{\text{argmin}}_{\k\in\mathbb{K}} MSE(\k|\mathcal{F}^N).
\end{align*}\normalsize


\subsection{Detailed Flowchart}\label{detail_workflow}


The detailed flow of the whole procedure is shown in Algorithm~\ref{algo_full}, where the function $select\_x(policy,...)$ uses the designated $policy$ to select an alternative to measure. The flowchart of our experiment decision $select\_x()$ is shown in Algorithm~\ref{algo_policy}. In $select\_x()$, besides the KGDP-$f$ and KGDP-$H$ policies, we also include three competing policies used in our experiments, namely (1) Pure Exploration, which chooses $x^n$ randomly; (2) Pure Exploitation, which always chooses the current best alternative; and (3) Max Variance (Max-Var), which picks the alternative that has the largest variance under prior belief.

\begin{algorithm}[!htp]
\caption{\small Flow of the resampling algorithm}\label{algo_full}{\small
\begin{algorithmic}[1]
\REQUIRE{Budgets: $N$; Alternatives: $\mathcal{X}=\{x_1,...,x_M\}$; Noise: $\sigma$; Large pool: $\mathbb{K}$; Smaller pool size: $R$; Number of candidates: $L$;;
Resample stepsize: $n^{resamp}$; Threshold for probability: $\epsilon$; and the $policy$ to select $x^n$.}
\ENSURE{Estimate of the optimal alternative: $\hat{x}^*$, estimate of parameters: $\hat{\k}^*$.}
\STATE{Choose $L$ samples out of $\mathbb{K}$ randomly.}
\STATE{Set $p_1^0=...=p_L^0=\frac{1}{L}$.}
\FOR{$n=0$ to $N-1$}
\STATE{$x^n=select\_x(policy, \mathcal{X}, \vec{p^n},\mathbb{L}^n$).}
\STATE{Take a measurement: $\hat{y}^{n+1}$.}
\STATE{Update each $p_i$ using: $$
p_i^{n+1}=\frac{\exp[-\frac{(\hat{y}^{n+1}-f_i(x))^2}{2\sigma^2}]}{\sum_{l=1}^L\exp[-\frac{(\hat{y}^{n+1}-f_l(x))^2}{2\sigma^2}]p_l^n}p_i^n.
$$}
\IF{$n+1\equiv 0 \bmod(n^{resamp})$ \textbf{or} $|\{p^{n+1}_i\geq \epsilon\}|\geq L/2$}
\STATE{Calculate MSE for all the $K$ samples.}
\STATE{Construct $\mathbb{S}^n$ as the set of $R$ $\k$'s with the smallest MSE.}
\STATE{Calculate the likelihood of each $\k$ in $\mathbb{S}^n$.}
\STATE{Construct $\mathbb{L}^n_{rm}$ as the set of $\k$'s to be removed, and set $p^{n+1}_i$ = 0 for $\k_i\in\mathbb{L}^n_{rm}$.}
\WHILE{$\min(p^{n+1})\leq \epsilon$}
\STATE{Find candidates with $p^{n+1}\leq\epsilon$ and remove them.}
\STATE{Select $|\mathbb{L}^n_{rm}|$ $\k$'s from $\mathbb{S}^n$ by weighted sampling without replacement according to their likelihoods. }
\STATE{Update posterior probabilities for all the $L$ new $\k$'s:
$$
p^{n+1}_i=\frac{\prod_{j=0}^{n}\exp[-\frac{(\hat{y}^{j+1}-f(x^j;\k_i^{n+1}))^2}{2\sigma^2}]}{\sum_{l=1}^L\prod_{j=0}^{n-1}\exp[-\frac{(\hat{y}^{j+1}-f(x^j;\k_l^{n+1}))^2}{2\sigma^2}]}.
$$}
\ENDWHILE
\ENDIF
\ENDFOR
\RETURN {$\hat{x}^* = \mathop{\text{argmax}}_{x\in\mathcal{X}}\sum_{i=1}^L p_i^N f(x;\k^N_i), \hspace{0.5em} \hat{\k}^*=\mathop{\text{argmin}}_{\k\in\mathbb{K}} MSE(\k|\mathcal{F}^N)$.}
\end{algorithmic}}
\end{algorithm}

\begin{algorithm}[!htp]
\caption{\small Choose the next alternative $x^n$.}\label{algo_policy}{\small
\begin{algorithmic}[1]
\REQUIRE{Policy: \textit{Pure Exploration, Pure Exploitation, Max-Var, KGDP-$f$, KGDP-$H$}; Alternatives: $\mathcal{X}=\{x_1,...,x_M\}$; Probability vector $\vec{p^n} = \{p_1^n,...,p_L^n\}$; $L$ candidates $\mathbb{L}^n=\{\k_1^n,...,\k_L^n\}$}.
\ENSURE{The alternative to measure next: $x^n$.}
\SWITCH{$policy$}
\CASE{KGDP-f}
\STATE{$x^n = \mathop{\text{argmax}}_{x\in\mathcal{X}}\nu^{KGDP-f,n}(x)$.}
\vspace{0.5em}
\ENDCASE
\CASE{KGDP-H}
\STATE{$x^n = \mathop{\text{argmax}}_{x\in\mathcal{X}}\nu^{KGDP-H,n}(x).$}
\vspace{0.5em}
\ENDCASE
\CASE{Pure Exploration}
\STATE{$x^n=rand(\{x_1,...,x_M\})$.}
\vspace{0.5em}
\ENDCASE
\CASE{Pure Exploitation}
\STATE{$x^n=\mathop{\text{argmax}}_{x\in\mathcal{X}}\bar{f}^n(x)$.}
\vspace{0.5em}
\ENDCASE
\CASE{Max-Var}
\STATE{$x^n=\mathop{\text{argmax}}_{x\in \mathcal{X}} \sum_{j=1}^L p_j[f(x;\k_j^n)-\bar{f}^n(x)]^2$.}
\vspace{0.5em}
\ENDCASE
\ENDSWITCH
\RETURN $x^n$.
\end{algorithmic}}
\end{algorithm}



\section{Convergence of KGDP without Resampling}\label{sec5}

By construction, KGDP-$f$ and KGDP-$H$ are naturally the optimal myopic policy to find the optimal alternative and parameter (in terms of entropy) respectively. In this section, we show that KGDP-$f$ and KGDP-$H$ with truth from prior are asymptotically optimal in the non-resampling scenario, a result that was not shown in \cite{Chen2014} where KGDP was first introduced. In other words, as the budget $N\rightarrow \infty$, KGDP-$f$ and KGDP-$H$ will both find the optimal alternative and the correct parameters.

We denote the finite set of alternatives as $\mathcal{X}$. Define $(\Omega, \mathcal{F}, \mathbb{P})$ as the probability space, where $\Omega$ is the set of all possible measurement histories $\{(x^0,\hat{y}^1),...,$ $(x^{N-1},\hat{y}^N)\}$, and $\mathcal{F}$ is the $\sigma$-algebra generated by the history. In this section, $N=\infty$.

To prove asymptotic optimality, we first show that if an alternative $x$ is measured infinitely often, we can learn its true function value almost surely. As $N$ goes to infinity, there will be a subset of alternatives that are evaluated infinitely often. We prove that under KGDP-$f$ or KGDP-$H$ policy with truth from prior, $f(x,\k^*)$ is the only function that fits the true function values of this subset. To show this, we reveal the nonnegativity of KGDP-$f$ and KGDP-$H$ scores, and the fact that any alternative measured infinitely often has its KGDP-$f$ and KGDP-$H$ scores converge to $0$. Then we use proof by contradiction. Assume $f(x,\k^*)$ is not the only function inferred by the infinitely measured alternatives. Then there exists at least an $x$ measured for a finite number of times. We claim that it has positive KGDP-$f$ and KGDP-$H$ scores in the limit. This is contrary to the fact that either KGDP-$f$ or KGDP-$H$ policy chooses the alternative with the largest score.

\bigskip

We begin by showing that if we measure a point $x$ infinitely often, the correct function value at $x$ will be found almost surely: 

\begin{lemma}\label{p_is_0}
Let $N^n(x)$ be the number of measurements taken on $x$ when the total number of measurements is $n$. If $N^n(x)\rightarrow\infty$ as $n\rightarrow\infty$, then for $\forall l\in\{1,2,...,L\}$ such that $f(x;\k_l)\ne f(x;\k^*)$, $p_l^n\rightarrow 0$ almost surely, i.e, $\mathbb{P}(\lim_{n\rightarrow \infty}p_l^n=0)=1$. 
\end{lemma}

The rigorous proof of Lemma~\ref{p_is_0} comes from the strong law of large numbers (shown in \hyperref[append]{Appendix}).

We define the set of the almost sure events as $\Omega_0$. For any $\omega\in\Omega$, we denote the set of alternatives as $\mathcal{X}_\infty(\omega)$ that are measured infinitely often.  Hence, for any $\omega\in\Omega_0$ and any $x\in\mathcal{X}_\infty(\omega)$, we have $p^n(\k^*)(\omega)\rightarrow 1$.

We hope that we can learn $\k^*$ via $\mathcal{X}_\infty$. To achieve this goal, we first study some properties of KGDP-$f$ and KGDP-$H$ scores. We start by showing that the value of information for both objectives is always nonnegative:
\begin{lemma}\label{f_non_neg}
For $\forall n\geq 0, \forall x\in \mathcal{X}$, the KGDP-$f$ score $\nu^{KGDP-f,n}(x)\geq 0$. Equality holds if and only if (1) there exists $x'$ such that $x'\in \mathop{\text{argmax}}_x f(x;\k_i)$ for all $i$ such that $p_i^n>0$, or (2) all functions with $p^n>0$ have the same value at $x$.
\end{lemma}

\begin{lemma}\label{H_non_neg}
For $\forall n\geq 0, \forall x\in \mathcal{X}$, the KGDP-$H$ score $\nu^{KGDP-H,n}(x)\geq 0$. Equality holds if and only if all functions with $p^n>0$ have the same value at $x$. 
\end{lemma}

\begin{proof}[Sketch of proof for Lemma~\ref{f_non_neg} and \ref{H_non_neg}]
(See \hyperref[append]{Appendix} for full proof.) First, we show $\mathbb{E}^n \left[p^{n+1}_i(x)\right] = p^n_i$. Then, applying Jensen's inequality will give us the nonnegativity of both KGDP-$f$ and KGDP-$H$.
\end{proof}


According to Lemma~\ref{f_non_neg} and Lemma~\ref{H_non_neg}, KGDP-$f$ equals $0$ in only two cases: (1) either when the functions are aligned, in which case $\nu^{KGDP,n}(x)=0$ for all $x$, (2) or when the functions with nonzero probabilities have the same value at $x$, in which case $\nu^{KGDP,n}(x)=0$ for this particular $x$. However, KGDP-$H$ equals $0$ only in the second case. 

We then show that for any $x$ measured infinitely often, its KGDP-$f$ and KGDP-$H$ scores also go to zero:

\begin{lemma}\label{v_score0}
For $\forall \omega\in\Omega_0$ and $x\in\mathcal{X}_\infty(\omega)$, we have $\lim_{n\rightarrow\infty}\nu^{KGDP-f,n}(x)(\omega)=0$, and $\lim_{n\rightarrow\infty}\nu^{KGDP-H,n}(x)(\omega)=0$.
\end{lemma}

Intuitively, as we measure $x$ infinitely often and learn the true value gradually, the condition for both KGDP-$f$ and KGDP-$H$ being $0$ is satisfied (by Lemma~\ref{f_non_neg} and \ref{H_non_neg}). The full proof, as shown in \hyperref[append]{Appendix}, reveals the fact that given a fixed $x$, $\nu^{KGDP-f}(x)$ and $\nu^{KGDP-H}(x)$ are both uniformly continuous in the probability vector space $\vec{p^n}=(p_1^n,...,p_L^n)$. Combined with the convergence result of $p_i^n$ shown by Lemma~\ref{p_is_0}, we conclude $\nu^{KGDP-f}(x)$ and $\nu^{KGDP-H}(x)$ also converge to $0$.

%
%


\bigskip

For any subset of alternatives, if it is sufficient to infer $\k^*$ by fitting the true function values of the subset, we call it a \textit{sufficient set}:
\begin{definition}
We define a \textit{sufficient set} $\mathcal{X}_s\subset\mathcal{X}$ as follows:  for $\forall \k\ne\k^*$, there exists $x\in\mathcal{X}_s$ such that $f(x;\k)\ne f(x;\k^*)$.
\end{definition}

In other words, a sufficient set $\mathcal{X}_s$ is a subset of $\mathcal{X}$ upon which we can distinguish $\k^*$ from all the others. Alternatively, for any subset $\mathcal{X}_s = \{x_1,...,x_m\}\subset \mathcal{X}$, if $f(x;\k^*)$ is the only function that fits the $m$ points $(x_1,f(x_1;\k^*)), ...,$ $(x_m,f(x_m;\k^*))$, then $\mathcal{X}_s$ is a sufficient set. Obviously, the largest sufficient set is $\mathcal{X}$. An example of an \textit{insufficient set} is that, if all functions have the same values at $\{x_1,...,x_m\}$, then this is not a sufficient set.

We show that under KGDP-$f$ or KGDP-$H$ policy, we measure a sufficient set almost surely:
\begin{lemma}\label{measure_suf1}
For any $\omega\in\Omega_0$, the alternatives measured infinitely often under the KGDP-$f$ or KGDP-$H$ policy constitute a sufficient set.
\end{lemma}

\begin{proof}[Sketch of proof]
(See full proof in \hyperref[append]{Appendix}.) If $\mathcal{X}_\infty(\omega)$ is not a sufficient set, then at least another function other than $f(x;\k^*)$ fits $\mathcal{X}_\infty(\omega)$. Hence, there exists $x\notin\mathcal{X}_\infty(\omega)$ such that this function has different values from $f(x;\k^*)$. We can prove that for this $x$, $\lim_{n\rightarrow\infty}\nu^{KGDP-f,n}(x)(\omega)>0$, and $\lim_{n\rightarrow\infty}\nu^{KGDP-H,n}(x)(\omega)>0$. This is contrary to the fact that KGDP-$f$ or KGDP-$H$ will not measure any $x$ out of $\mathcal{X}_\infty(\omega)$ after a finite number of times.
\end{proof}

Using the previous lemmas, we can conclude with the following asymptotic optimality result:

\begin{theorem}\label{KGDP-f_con}
Non-resampling KGDP-$f$ with truth from prior is asymptotically optimal in finding both the optimal alternative and the correct parameter. The same holds for KGDP-$H$.
\end{theorem}

The detailed proofs of Lemma~\ref{f_non_neg} - \ref{measure_suf1} and Theorem~\ref{KGDP-f_con} are given in \hyperref[append]{Appendix}.

\section{Convergence of KGDP with Resampling}\label{sec6}

We now establish the important result that the resampling algorithm converges to the true optimal solution, both in terms of identifying the right parameter $\k^*$ as well as the optimal design $x^*$.

As we move from the non-resampling case to the resampling one, the difference is that the candidates are no longer fixed, and $\k^*$ is not guaranteed to be one of the candidates.

Note that the nonnegativity of the knowledge gradient established in Lemma~\ref{f_non_neg} and \ref{H_non_neg} still holds in the resampling case, regardless of whether the candidates include the truth or not.

Unlike the non-resampling case, we have two probability spaces here. One is still $(\Omega, \mathcal{F}, \mathbb{P})$ as above, defined on the measurement history, while the other one describes the randomness of selecting candidates in the resampling process. We denote the latter as $(\Omega_\k, \mathcal{F}_\k, \mathbb{P}_\k)$, where $\Omega_\k$ is the set of all possible selection combinations of the $L$ candidates in $N$ iterations (in this section, $N=\infty$). The full probability space, denoted as $(\Omega_f, \mathcal{F}_f, \mathbb{P}_f)$, is the product space of $(\Omega, \mathcal{F}, \mathbb{P})$ and $(\Omega_\k, \mathcal{F}_\k, \mathbb{P}_\k)$.  We use $\omega$ to represent an element in $\Omega$ and $\omega_\k$ in $\Omega_\k$.




Let $\mathbb{K}$ denote the large pool, i.e, the set of the $K$ $\k$'s. We can adapt the concept of a sufficient set by considering the entire set $\mathbb{K}$. For a subset of alternatives $\mathcal{X}_s\subset\mathcal{X}$, if $f(x;\k^*)$ is the only function among all the $K$ functions that fits the true values of alternatives in $\mathcal{X}_s$, we call it a sufficient set.

Unlike the non-resampling case, since the candidates may keep changing, the limit of a particular $p_i^n$ may not exist. However, if we regard the whole $\mathbb{K}$ as the candidate set and consider the probability of each $\k$, then Lemma~\ref{p_is_0} still holds, demonstrating that infinite measurements on $x$ give its true value. We define  $\Omega_1$ as the set of almost sure events in $\Omega$ for Lemma~\ref{p_is_0} to hold.



\begin{lemma}\label{measure_suf}
For any $\omega\in\Omega_1$, the alternatives measured infinitely often under the KGDP-$f$ or KGDP-$H$ policy constitute a sufficient set. We denote this set as $\mathcal{X}_\infty(\omega)$.
\end{lemma}

\begin{proof}[Sketch of proof]
(See \hyperref[append]{Appendix} for full proof.) Assume the contrary. Then there exists at least one $\k\ne\k^*$, such that $f(x;\k)=f(x;\k^*)$ for any $x\in\mathcal{X}_\infty(\omega)$. Denote the set of such $\k$'s as $\mathbb{K}'$. We can show that for this fixed $\omega$, the event happens infinitely often that at least two $\k$'s out of $\mathbb{K}'\bigcup \{\k^*\}$ are included in the candidate set. At these times, we can further show that there exists an $x\notin \mathcal{X}_\infty(\omega)$ and $\e>0$ such that $\nu^{KGDP-f,n}(x)>\e$ and $\nu^{KGDP-H,n}>\e$, while on the other hand, for any $x\in\mathcal{X}_\infty(\omega)$, $\lim_{n\rightarrow\infty}\nu^{KGDP-f,n}(x)=0$ and $\lim_{n\rightarrow\infty}\nu^{KGDP-H,n}(x)=0$. This is contradictory to the fact that any $x\notin\mathcal{X}_\infty(\omega)$ is measured for a finite number of times.
\end{proof}



Lemma \ref{measure_suf} implies the following theorem, which is also our main result in this section:

\begin{theorem}\label{re_KGDP-f_con}
Resampling KGDP-$f$ is asymptotically optimal in finding the optimal alternative and the correct parameter. The same holds for KGDP-$H$.
\end{theorem}

\begin{proof}[Sketch of proof]
(See \hyperref[append]{Appendix} for full proof.) The crucial step is to show that for any $\omega\in\Omega_1$, there exists time $T(\omega)$, such that $\k^*$ always appears in the candidate set after $T(\omega)$ with probability $1$. Note that this probability is calculated in the $(\Omega_\k, \mathcal{F}_\k, \mathbb{P}_\k)$ space. Extending this result to the full probability space $(\Omega_f, \mathcal{F}_f, \mathbb{P}_f) = (\Omega, \mathcal{F}, \mathbb{P})\times(\Omega_\k, \mathcal{F}_\k, \mathbb{P}_\k)$, we can see that $\lim_{n\rightarrow\infty}p^n(\k^*)=1$ happens with probability $1$.
\end{proof}

\section{Empirical Results}\label{sec7}

We study an asymmetric unimodal problem as a benchmark function, and another material science problem, known as the nanoemulsion problem in a water-oil-water (W/O/W) model.

We mainly demonstrate results from the following aspects:
\begin{itemize}
 \item The visualization of the resampling process, including the evolution of the candidates, small pool, and approximated function $\bar{f}^n(x)$;
 \item The performance of resampling KGDP-$f$ and KGDP-$H$, both in relation to each other and their absolute performance in terms of different metrics;
 \item The relative performance of resampling KGDP-$f$ and KGDP-$H$ policies versus competing policies;
 \item The comparison of resampling and non-resampling KGDP methods;
 \item The empirical rate of convergence under different noise and dimensionality settings.
\end{itemize}

We feel that the rate of convergence result is the most important.

%

\subsection{An Asymmetric Unimodal Benchmark Function}

We study a multidimensional benchmark function, which produces a wide range of asymmetric, unimodular functions with heteroscedastic noise, which we have found to be a challenging class of test problems. The function is given by
\small\begin{align*}
f(x_1,...,x_k) = \sum_{i=1}^k\eta_{1,i}\mathbb{E}[\min(x_i,(D-\sum_{j=1}^{i-1}x_j)^+)]-\sum_{i=1}^k\eta_{2,i}x_i,
\end{align*}\normalsize
where $x_1,...,x_k$ are the decision vector, $D$ is a uniformly distributed random variable, $\eta_{1,i}$ are fixed constants, and $\eta_{2,i}$ are unknown parameters. In the following experiments, we fix the variance of D, but set its mean as well as $\eta_{2,i}$'s as parameters ($\k$), which results in a $(k+1)$-dimensional problem.

\subsubsection{Illustration of the Resampling Process}

We first show the results of resampling KGDP-$f$ in two-dimensional cases, where we can plot the $\k$ space.

\begin{figure}[!bp]
\centering
\hspace{3.5em}\includegraphics[width=0.98\textwidth]{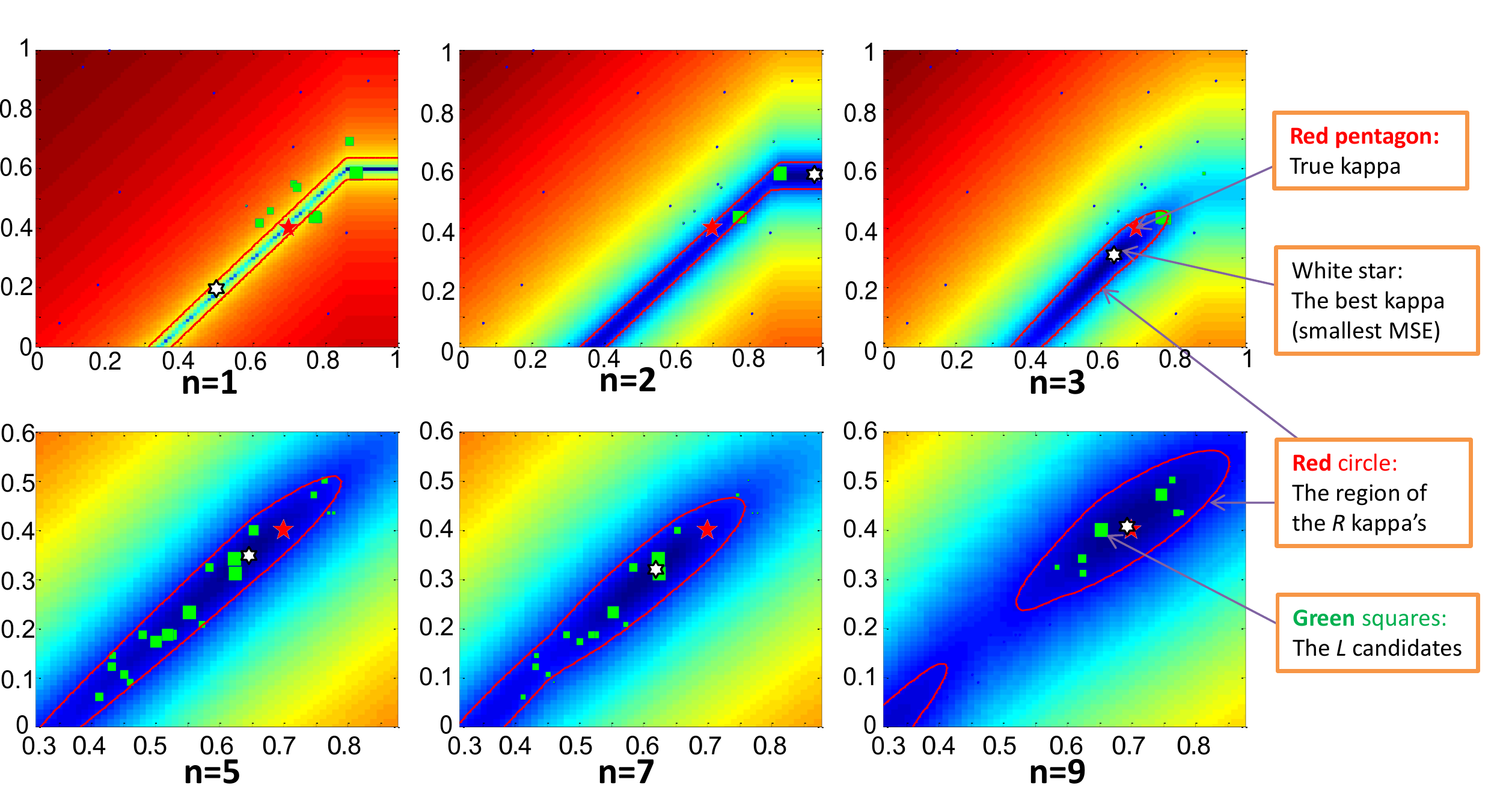}
\caption{Illustration of resampling in $\k$ space (Iterations $1$, $2$, $3$, $5$, $7$, $9$) ($5$,$7$,$9$ are zoomed in).}\label{fig_epigraph1}
\end{figure}

Figure~\ref{fig_epigraph1} shows the evolution of the small pool and the $L$ candidates in $\k$ space in the first several iterations of a realization of the problem. In these images, the horizontal and vertical axes represent $\k_1$ and $\k_2$ respectively, with color indicating the values of MSE. The red pentagon shows where the truth $\k^*$ is located, while the white star indicates the $\k$ with the smallest MSE in the large pool. The region circled by the red line is the range of the small pool.  The green squares indicate the locations of the candidates, and the sizes are proportional to their probabilities. Note that these images have been scaled to show more detail. We can see that the small pool shrinks quickly towards around $\k^*$, and the $\k$ with the smallest MSE as well as the candidates also converges to $\k^*$ within only a few iterations.

\subsubsection{Comparison of Different Policies to Choose $x^n$}

We use the opportunity cost (OC) to evaluate the performance of various policies from the alternative perspective. OC is defined as the difference between the values at our estimated optimal $x$ and the true optimal $x$, i.e.,
\small\begin{align*}
OC(n) = \max_{x\in \mathcal{X}}f(x;\k^{*}) - f(\mathop{\text{argmax}}_{x\in\mathcal{X}} \bar{f}^n(x);\k^{*}).
\end{align*}\normalsize

To normalize the opportunity cost, we define as the percentage OC the ratio with respect to the optimal function value, i.e.:
\small\begin{align*}
OC\%(n)=\frac{\max_{x\in \mathcal{X}}f(x;\k^{*}) - f(\mathop{\text{argmax}}_{x\in\mathcal{X}} \bar{f}^n(x);\k^{*})}{\max_{x\in \mathcal{X}}f(x;\k^{*})}.
\end{align*}\normalsize


We define the noise level as the ratio of the noise standard deviation to the range of the true function, i.e, $\frac{\sigma}{\max_x f(x;\k^*)-\min_x f(x;\k^*)}$. When we say the noise is $20\%$, we mean this ratio is $20\%$.

Figure~\ref{fig_3D_OC} shows the comparison of the five policies  in three dimensional settings. The subimages correspond to $5\%$, $20\%$ and $50\%$ of noise from left to right. We can see that KGDP-$f$ and KGDP-$H$ are among the best under different noise levels. We have similar results for two- and five-dimensional cases, too.

%
%
%
\begin{figure}[!htp]
\centering
\includegraphics[width=0.95\textwidth]{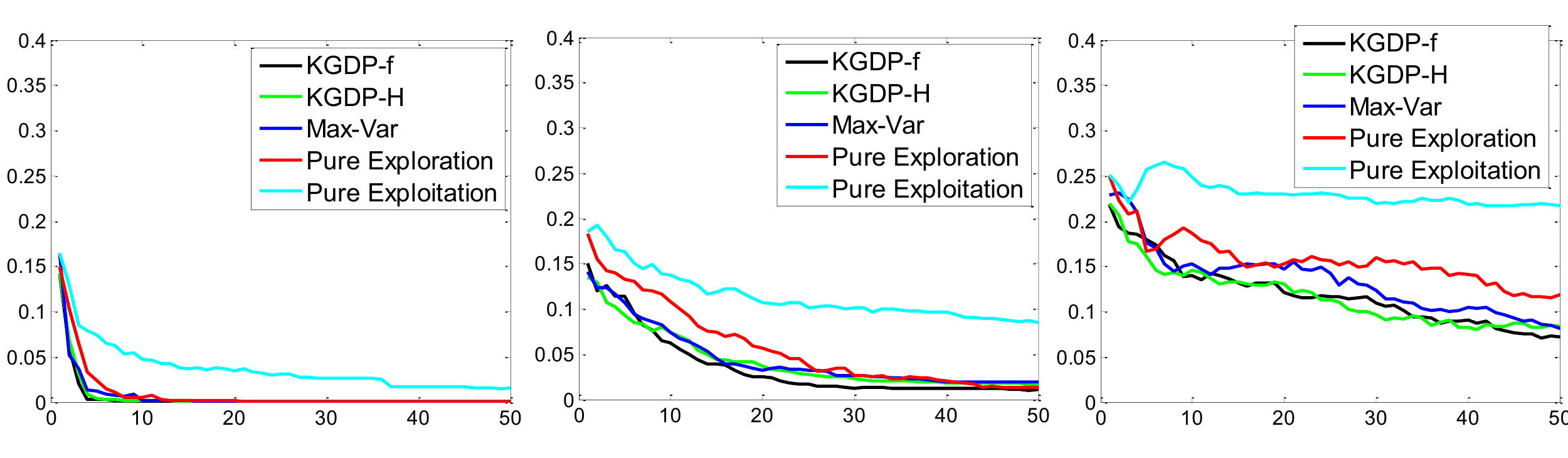}
\caption{\small OC\% of different policies for three-dimensional $\k$ ($5\%$, $20\%$ and $50\%$ noise from left to right).}\label{fig_3D_OC}
\end{figure}


%

\subsection{Application in W/O/W Nanoemulsion Stability Study}


We study the same problem of water-oil-water(W/O/W) nanoemulsion stability as in \cite{Chen2014}, and repeat its experiments using the original non-resampling KGDP and our resampling KGDP methods to compare the results.

The problem studies controlled payload delivery using water-oil-water (W/O/W) double nanoemulsions, which depicts the delivery of payload molecules from the internal water droplet, surrounded by a larger droplet of oil, to the external water. This process is performed via laser excitation and controlled by several experimental variables. We can model the stability of this process as a function $f(x;\k)$ that is nonlinear in $\k$ (as well as $x$).  $x$ is a five-dimensional control variable and includes variables such as the initial volume of the external water, the initial volume fraction of the oil droplets, and the diameter of the oil droplets. $\k$ includes seven dimensions, representing the unknown parameters that appear in the formulation, such as the energy barrier for flocculation, the rate prefactor for coalescence and droplet adsorption/desorption energy barrier difference. Our goal is to conduct a series of experiments to learn the optimal setting ($x$) to achieve the best stability, and also to learn the correct unknown parameters $\k$. A more detailed introduction and formulation of this problem can be found in \cite{Chen2014}.

In order to be consistent with \cite{Chen2014}, we fix three dimensions of $x$ and create a finite set of alternatives using the other two dimensions. For $\k$, we study 1) a three dimensional case and 2) a seven dimensional one in our experiments. Remember that the dimensionality of $\k$ is more important than the dimensionality of $x$, since we always have finite samples of $x$.

\subsubsection{Three Dimensional $\k$}

Figure~\ref{fig_nano_3D_OC_50} shows the comparison of the various policies in terms of OC\% under $50\%$ of noise. The left image is a comparison across the five policies, while the right one compares the resampling method with non-resampling KGDP in \cite{Chen2014}.

\begin{figure}[!htp]
\centering
\begin{subfigure}{0.48\textwidth}
\centering
\includegraphics[width=0.8\textwidth]{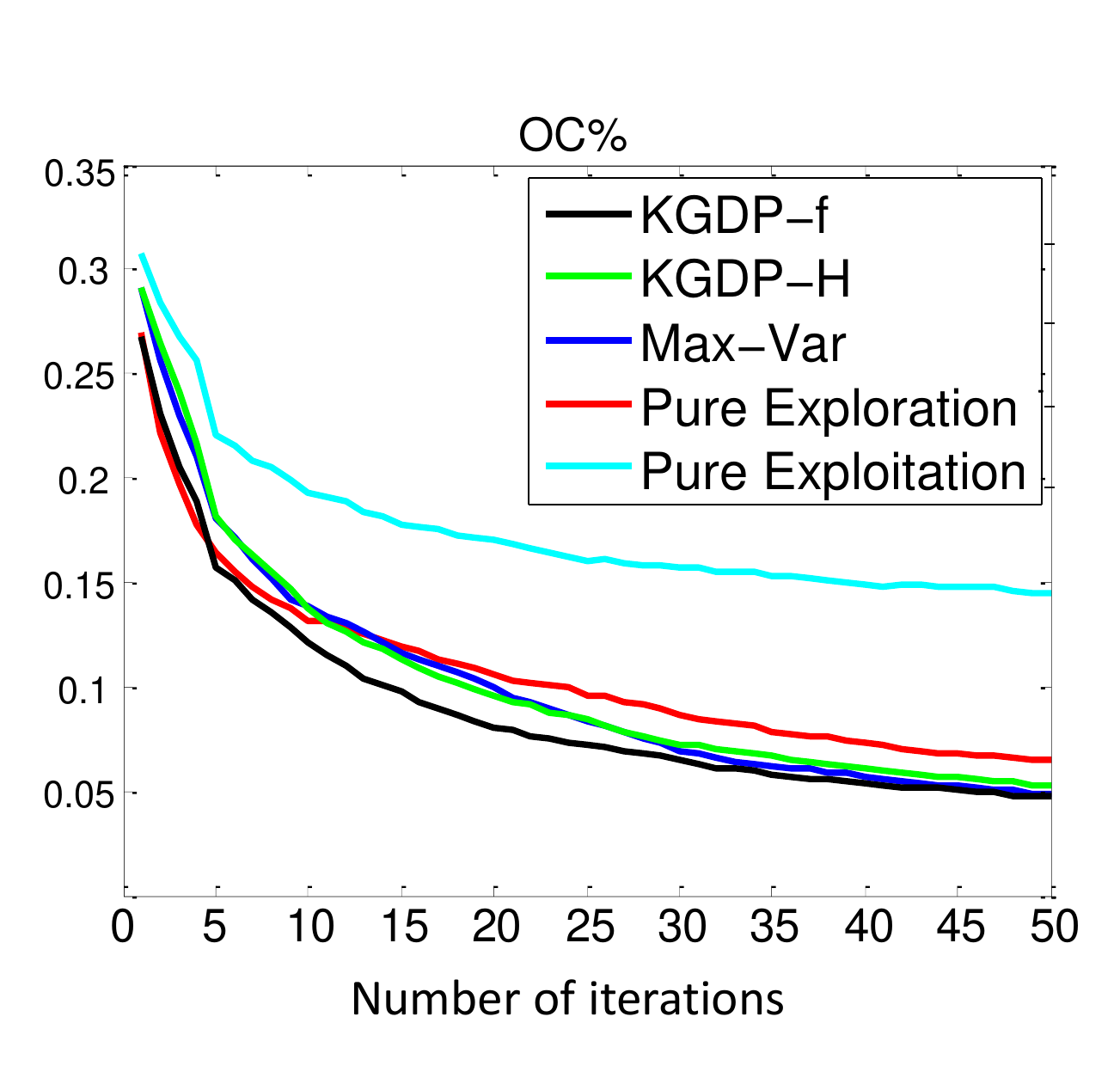}
\caption{\small Comparison across five policies}
\end{subfigure}
\begin{subfigure}{0.48\textwidth}
\centering
\vspace{0.4em}
\includegraphics[width=0.76\textwidth]{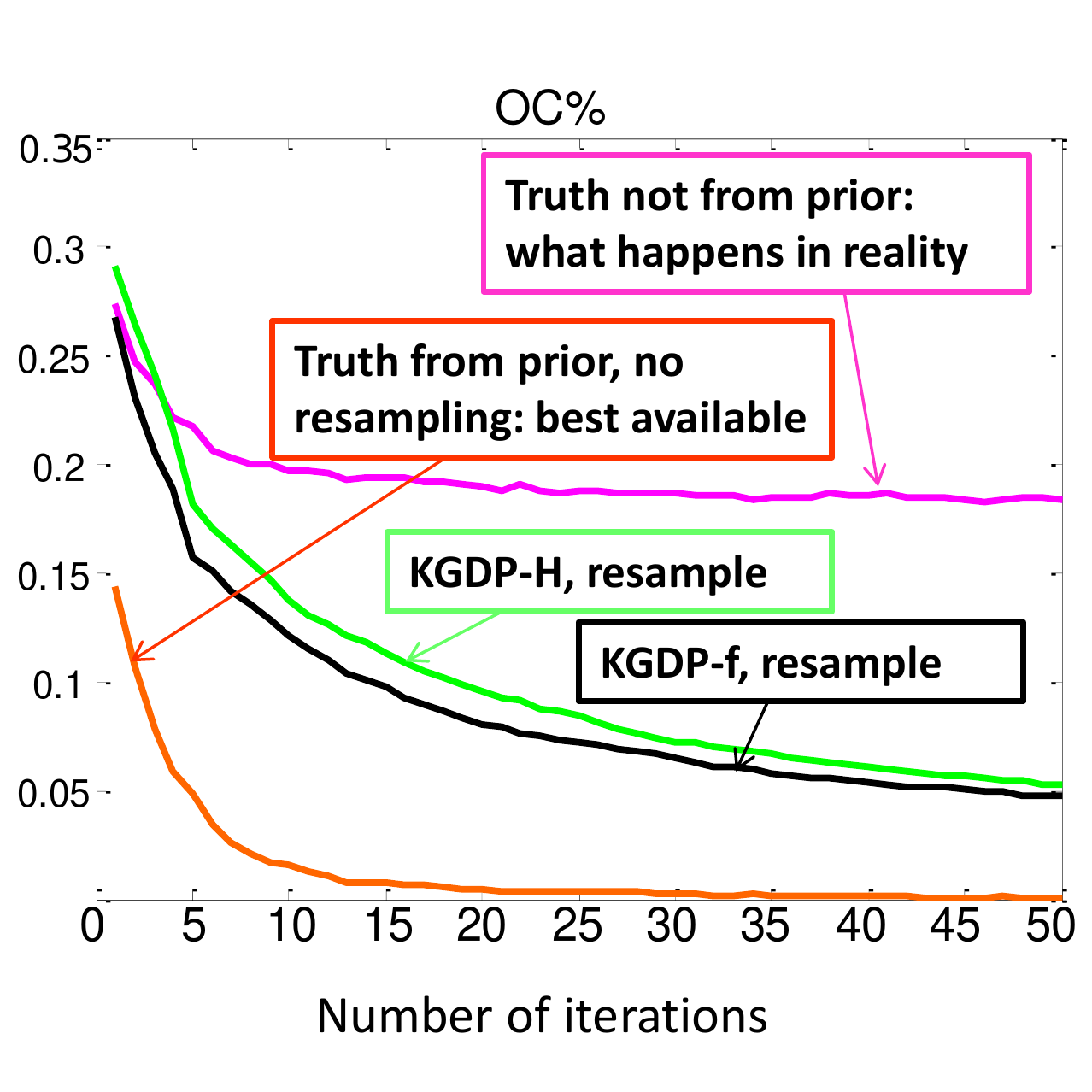}
\caption{\small Comparison with non-resampling KGDP}
\end{subfigure}
\caption{\small OC\% of different policies in three-dimensional nanoemulsion ($50\%$ noise).}\label{fig_nano_3D_OC_50}
\end{figure}


%
%

The left image indicates that Max-Var, KGDP-$f$ and KGDP-$H$ are among the best, showing very close performance. The right image compares: (1) the orange curve: the non-resampling KGDP-$f$ with truth from prior, which is the ideal case; (2) the pink curve: the non-resampling KGDP-$f$ with truth not from prior, which is what happens with truth not from prior using \cite{Chen2014}; (3) the black curve: resampling KGDP-$f$; and (4) the green curve: resampling KGDP-$H$. We can see that (2) gets stuck after a few iterations, while our resampling methods using either KGDP-$f$ or KGDP-$H$ are both able to find the more probable parameters and approach the ideal case gradually.



We have shown the performance of our resampling algorithm in finding the best alternative in terms of OC\%. To show its performance in learning the parameters, we denote our best estimate of $\k^*$ as $\hat{\k}^*$, and calculate:
\begin{itemize}
\item[(1)] the mean square error of $f(x;\k^*)$ and $f(x;\hat{\k}^*)$, defined as (remember $M$ is the size of $\mathcal{X}$):
    \small\begin{align*}
    f_{MSE} = \frac{1}{|\mathcal{X}|}\sum_{x\in\mathcal{X}}|f(x;\k^*)-f(x;\hat{\k}^*)|^2,
    \end{align*}\normalsize
\item[(2)] or the error of a particular dimension of $\k$, say $|\k^*_1-\hat{\k}^*_1|$.
\end{itemize}

Note that in method (2), we do not calculate the error of all dimensions as $||\k^*-\hat{\k}^*||$. This is because different dimensions of $\k$ may have different units and follow various distributions, hence it is meaningless to sum up the errors of different dimensions. Moreover, in a higher dimensional problem, some dimensions may converge more quickly than others, and while the entire vector converges in aggregate, individual dimensions may converge more slowly than others.

We show the MSE and error of $\k$ in Figure~\ref{fig_nano_3D_mse_50}. In the first image, we calculate $f_{MSE}$, and plot $\frac{\sqrt{f_{MSE}}}{\max_x f(x;\k^*)-\min_x f(x;\k^*)}$. and the other three show $|\k^*_1-\hat{\k}^*_1|$, $|\k^*_2-\hat{\k}^*_2|$ and $|\k^*_3-\hat{\k}^*_3|$, respectively. Note that in this example, the second dimension seems less important than the other two. We also have similar results under $20\%$ of noise.  

\begin{figure}[!htp]
\centering
\includegraphics[width=\textwidth]{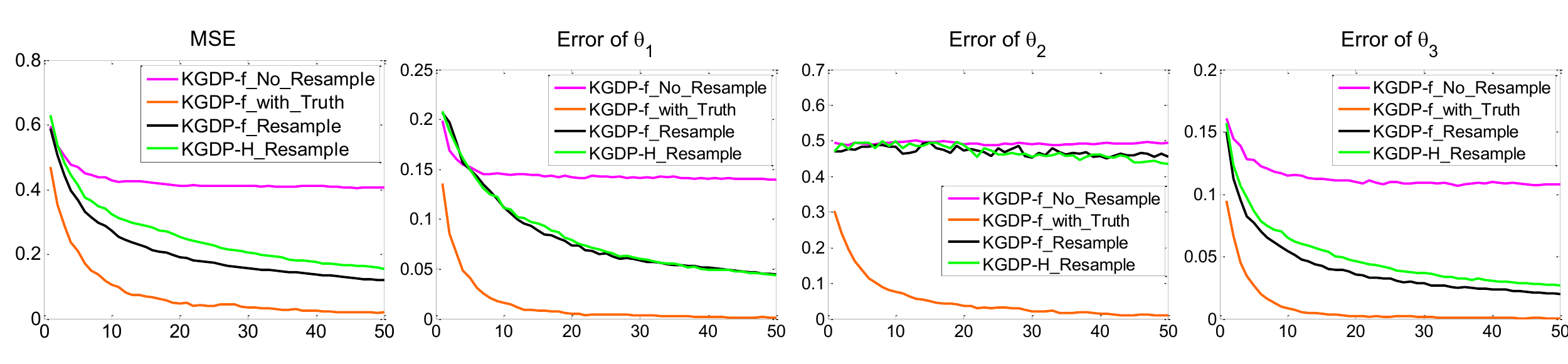}
\caption{\small Performance of resampling KGDP in finding $\k$ in three-dimensional nanoemulsion ($50\%$ noise).}\label{fig_nano_3D_mse_50}
\end{figure}

%
%


\subsubsection{Seven Dimensional $\k$}

In a seven-dimensional case, Figure~\ref{fig_nano_7D_fbar} shows the process of $\bar{f}^n(x)$ approaching the true function, using resampling KGDP-$f$ under $20\%$ noise. The four images in the second row are our approximations at $n=1,10,20,45$ respectively.


%

\begin{figure}[!htp]
\centering
\includegraphics[scale=0.4]{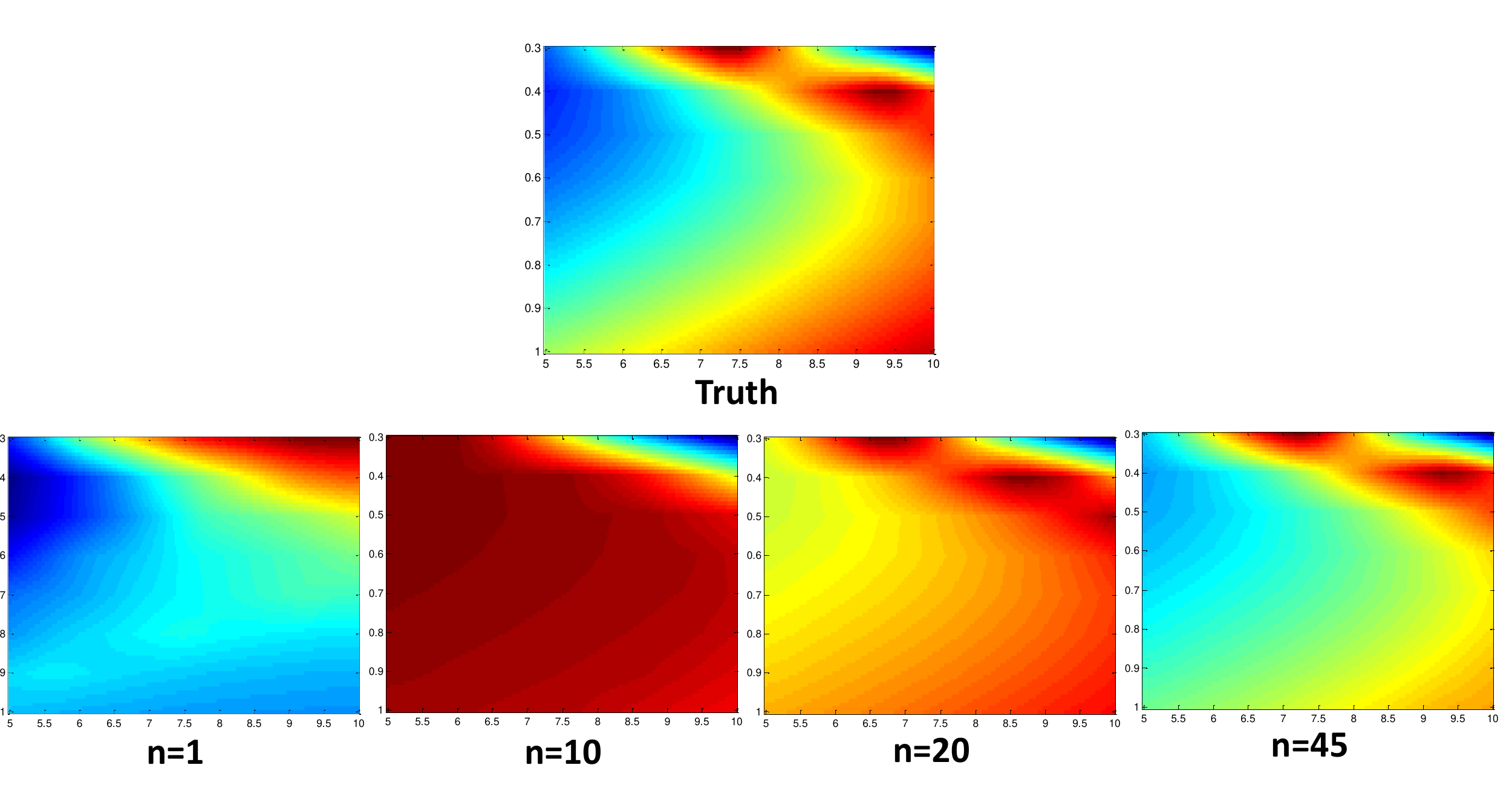}
\caption{\small Evolution of $\bar{f}^n(x)$ (seven-dimensional $\k$, $20\%$ noise).}\label{fig_nano_7D_fbar}
\end{figure}

The result of OC\% under $50\%$ noise are shown in Figure~\ref{fig_nano_7D_OC_50}. We have similar results regarding learning the parameters as in three-dimensional cases. 

\begin{figure}[!htp]
\centering
\begin{subfigure}{0.48\textwidth}
\centering
\includegraphics[width=0.7\textwidth]{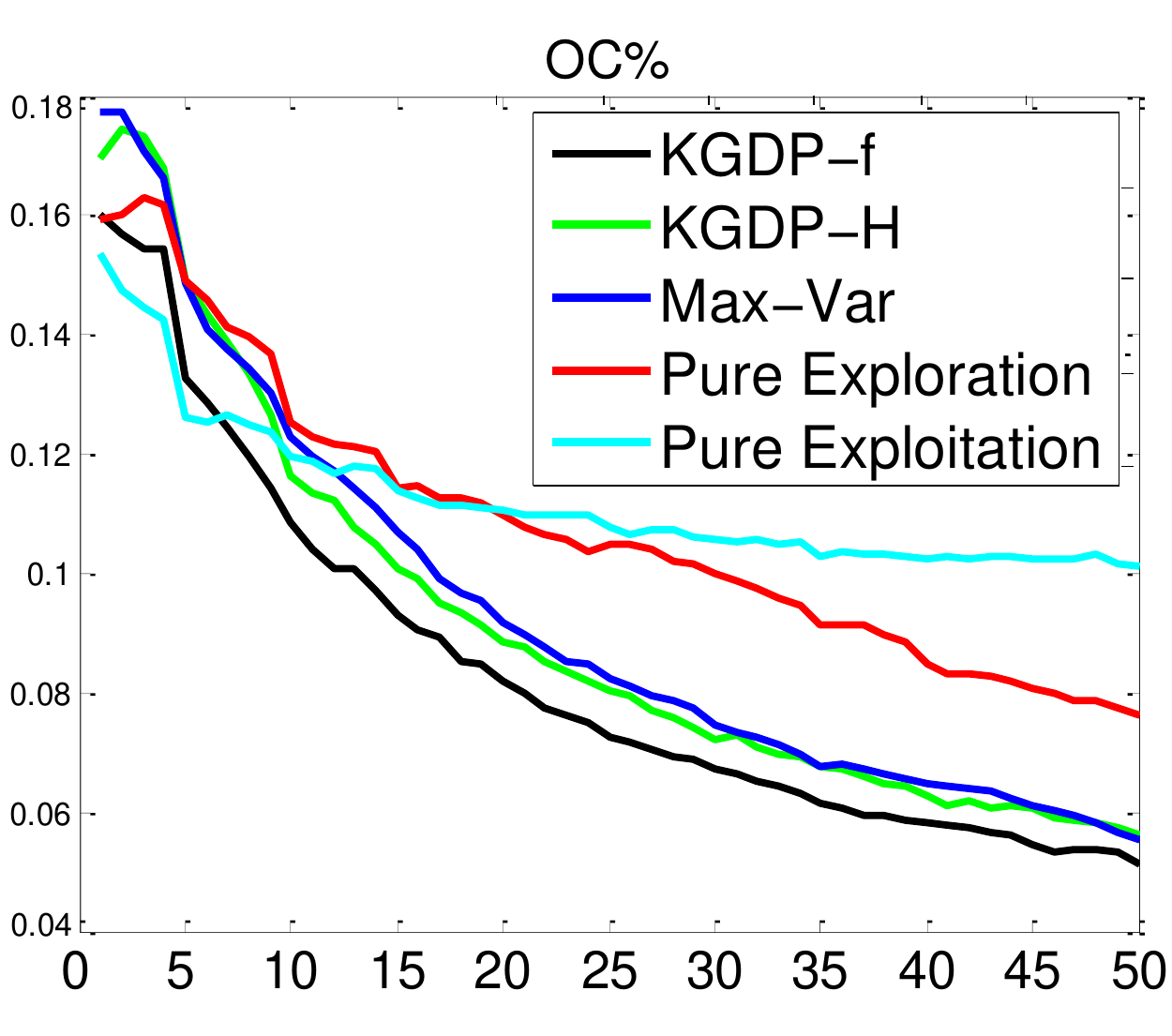}
\caption{\small Comparison across five policies}
\end{subfigure}
\begin{subfigure}{0.48\textwidth}
\centering
\includegraphics[width=0.7\textwidth]{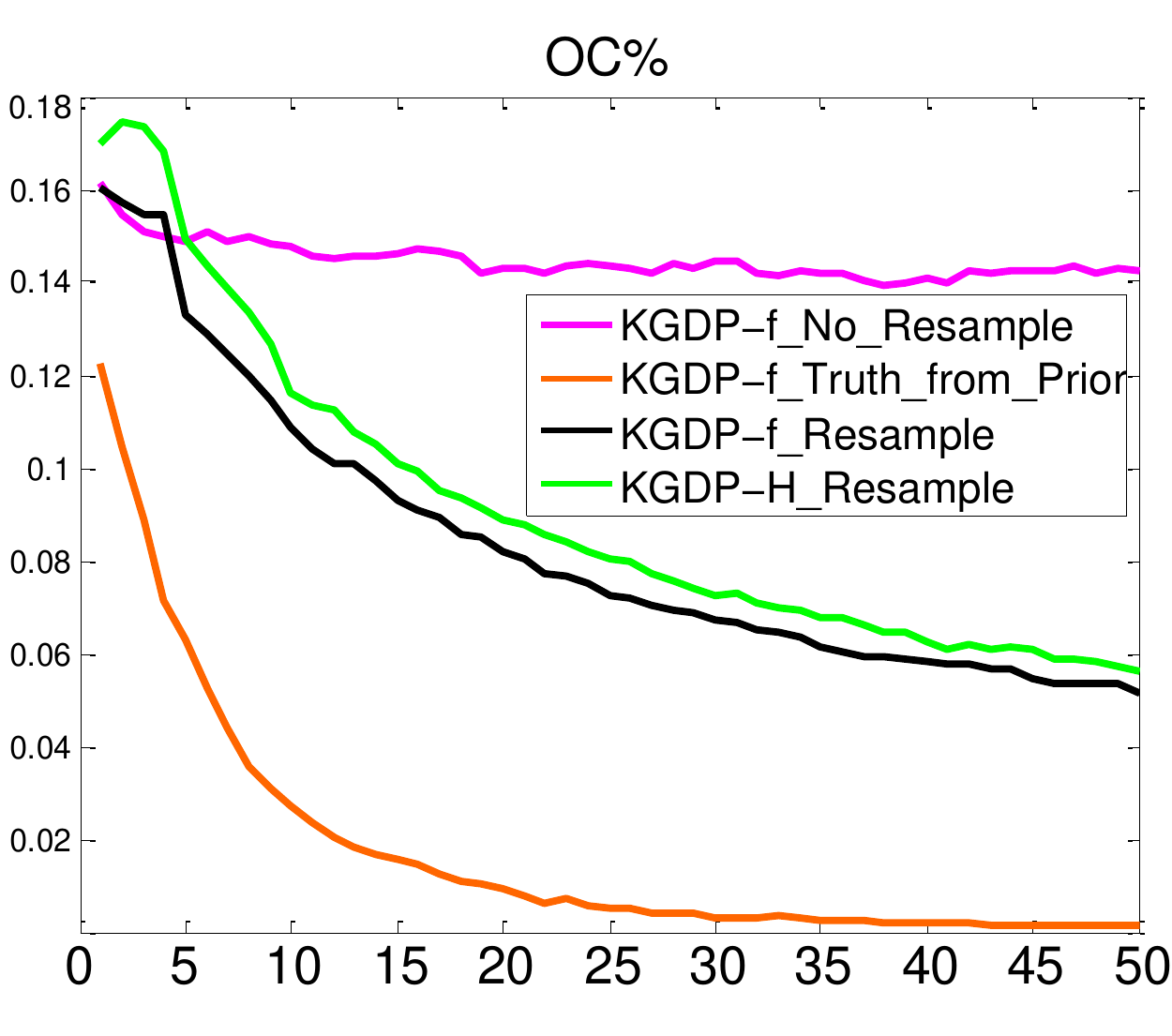}
\caption{\small Comparison with non-resampling KGDP}
\end{subfigure}
\caption{\small OC\% of different policies in seven-dimensional nanoemulsion ($50\%$ noise).}\label{fig_nano_7D_OC_50}
\end{figure}

\section{Conclusion}\label{sec8}

Optimal learning for nonlinear belief models has been an interesting but difficult problem. With a belief model nonlinear in unknown parameters, the knowledge gradient becomes computationally intractable if using continuous paramenter space. In \cite{Chen2014}, the Knowledge Gradient with Discrete Priors (KGDP) model is proposed, which uses a sampled representation of the parameter space. This method works well when one of the sampled candidates is exactly the truth, but this would not happen in practice.

In this paper, we propose a resampling algorithm that cooperates with KGDP to solve this problem. We no longer require that one of the candidates is correct, but use the measurement history to guide resampling and discovering more probable parameters in the parameter space. Within a small budget of noisy and expensive experiments, our algorithm is capable of both finding the optimal alternative that maximizes the value function, and narrow down the location of the unknown parameters. Experiments on both a synthetic benchmark function and a real world materials science problem, known as the stability of W/O/W nanoemulsion problem, have shown strong performance of our algorithm.

We recognize the duality of this optimization problem and manipulate in both the alternative and the parameter spaces. Besides the idea of maximizing the value function (denoted as KGDP-$f$) as used in previous optimal learning literatures, we also put forward an entropy-oriented metric that works under the knowledge gradient framework, which aims directly at lowering the uncertainty of the parameters (in terms of entropy). Experiments show strong performance of both metrics.

We also prove the asymptotic optimality of KGDP-$f$ and KGDP-$H$ in both the non-resampling with truth from prior case, and the resampling framework. That means, given an infinite number of measurements, we are able to find both the optimal alternative and the correct parameter. Hence, KGDP-$f$ and KGDP-$H$ are both myopically and asymptotically optimal in learning the alternative and reducing parameter uncertainty respectively.


\setcounter{figure}{0}
\counterwithin*{figure}{section}
\vspace{0.5em}
\appendix
\section*{Appendix (Proofs)}\label{append}
\vspace{1em}

\noindent\textbf{Lemma~\ref{p_is_0}. }
Let $N^n(x)$ be the number of measurements taken on $x$ when the total number of measurements is $n$. If $N^n(x)\rightarrow\infty$ as $n\rightarrow\infty$, then for $\forall l\in\{1,2,...,L\}$ such that $f(x;\k_l)\ne f(x;\k^*)$, $p_l^n\rightarrow 0$ almost surely, i.e, $\mathbb{P}(\lim_{n\rightarrow \infty}p_l^n=0)=1$.

\subsubsection*{Proof of Lemma~\ref{p_is_0}}

We prove the following claim first.

{\bf Claim 1:} $\forall m\in \{1,...,M\}$, let $p_1^n, p_2^n, ..., p_L^n$ be the probabilities of the $L$ $\k$'s after $n$ measurements at $x_m$ (i.e, we take $n$ measurements only at $x$ and nowhere else). For $\forall \e>0$, $\forall l\in\{1,2,...,L\}$ such that $f(x_m;\k_l)\ne f(x_m;\k^*)$, $\mathbb{P}(p_l^n<\e)=1$, as $n\rightarrow \infty$.

{\bf Proof of Claim 1:}

If we take $n$ measurements at $x_m$, suppose the results are $\{\hat{y}^1,...,\hat{y}^n\}$. Let $f_l = f(x_m;\k_l)$. Then $p_1,...,p_L$ are given by:
\small
\begin{align*}
p^n_l=\frac{\prod_{j=1}^{n}\exp[-\frac{(\hat{y}^j-f_l)^2}{2\sigma^2}]}{\sum_{i=1}^L\prod_{j=1}^{n}\exp[-\frac{(\hat{y}^j-f_i)^2}{2\sigma^2}]}.
\end{align*}\normalsize

Let $f^*=f(x_m;\k^*)$. Then $\hat{y}^j\sim N(f^*, \sigma^2)$. Let $Z^j=\frac{\hat{y}^j-f^*}{\sigma}$, then $\hat{y}^j\sim f^*+\sigma Z^j$, $Z^j\sim N(0,1)$.

For any $l$ such that $f_l\ne f^*$,
\small
\begin{align*}
p_l^n \leq & \frac{\prod_{j=1}^{n}\exp[-\frac{(\hat{y}^j-f_l)^2}{2\sigma^2}]}{\prod_{j=1}^{n}\exp[-\frac{(\hat{y}^j-f^*)^2}{2\sigma^2}]}
= \frac{\exp[-\frac{\sum_{j=1}^n (\hat{y}^j-f_l)^2}{2\sigma^2}]}{\exp[-\frac{\sum_{j=1}^n (\hat{y}^j-f^*)^2}{2\sigma^2}]}
= \frac{\exp[-\frac{\sum_{j=1}^n (f^*-f_l+\sigma Z^j)^2}{2\sigma^2}]}{\exp[-\frac{\sum_{j=1}^n (Z^j)^2}{2}]}\\
= &\exp\left(-\frac{\sum_{j=1}^n [(f^*-f_l+\sigma Z^j)^2 - (\sigma Z^j)^2]}{2\sigma^2}\right)\\
= & \exp\left(-\frac{ n(f^*-f_l)^2 + 2\sigma (f^*-f_l)\sum_{j=1}^n Z^j }{2\sigma^2}\right).\numberthis\label{Eq_bound1}
\end{align*}\normalsize

For any $\e>0$, $\exp\left(-\frac{ n(f^*-f_l)^2 + 2\sigma (f^*-f_l)\sum_{j=1}^n Z^j }{2\sigma^2}\right)<\e \Rightarrow p_l^n<\e$, and therefore\\ $\{\lim_{n\rightarrow\infty}\exp\left(-\frac{ n(f^*-f_l)^2 + 2\sigma (f^*-f_l)\sum_{j=1}^n Z^j }{2\sigma^2}\right) = 0\} \subseteq \{ \lim_{n\rightarrow \infty}p_l^n = 0 \}$. We show that\small
\begin{align*}
\mathbb{P}\left[\lim_{n\rightarrow\infty} \exp\left(-\frac{ n(f^*-f_l)^2 + 2\sigma (f^*-f_l)\sum_{j=1}^n Z^j }{2\sigma^2}\right) = 0\right]=1,
\end{align*}\normalsize
which implies the almost sure convergence of $p_l^n$.

By the strong law of large numbers, $\frac{\sum_{j=1}^n Z^j}{n} \xlongrightarrow{a.s.} 0$. That is, there exists $\Omega_{x,l}\in \Omega$, such that for $\forall \omega\in\Omega_{x,l}$, $\lim_{n\rightarrow\infty}\frac{\sum_{j=1}^n Z^j(\omega)}{n}=0$, and $\mathbb{P}(\Omega_{x,l})=1$. Our goal is to show that for $\forall \e>0$, $\forall \omega\in\Omega_{x,l}$,\small
\begin{align*}
\lim_{n\rightarrow\infty} \exp\left(-\frac{ n(f^*-f_l)^2 + 2\sigma (f^*-f_l)\sum_{j=1}^n Z^j(\omega) }{2\sigma^2}\right) = 0.
\end{align*}\normalsize

Without loss of generality, we assume $f^*-f_l\geq 0$. For $\forall \omega\in \Omega_{x,l}$, there exists $N_1\in\mathbb{N}$ such that for $\forall n>N_1$, $\left| \frac{\sum_{j=1}^n Z^j(\omega)}{n} \right| < \frac{f^*-f_l}{4\sigma}$.

Therefore, for any $n>N_1$, $n\left[(f^*-f_l)^2 + 2\sigma (f^*-f_l)\frac{\sum_{j=1}^n Z^j(\omega)}{n}\right] > n\cdot \frac{(f^*-f_l)^2}{2}$. Hence, for $\forall n>N_1$,
\small
\begin{align}
\exp\left(-\frac{ n(f^*-f_l)^2 + 2\sigma (f^*-f_l)\sum_{j=1}^n Z^j(\omega) }{2\sigma^2}\right) < \exp\left(-\frac{n(f^*-f_l)^2}{4\sigma^2}\right).\label{Eq_bound2}
\end{align}
\normalsize

For $\forall \e>0$, there exists $N_2$ such that for $\forall n>N_2$, $\exp\left(-\frac{n(f^*-f_l)^2}{4\sigma^2}\right)<\e$.

Take $N=\max(N_1, N_2)$, we know that for $\forall \omega\in\Omega_{x,l}$, $\forall \e>0$, there exists $N>0$, such that for $\forall n>N$,
\small
\begin{align*}
\exp\left(-\frac{ n(f^*-f_l)^2 + 2\sigma (f^*-f_l)\sum_{j=1}^n Z^j(\omega) }{2\sigma^2}\right) < \e.
\end{align*}
\normalsize

Hence,
\small
\begin{align*}
\Omega_{x,l}\subseteq \{ \lim_{n\rightarrow\infty} \exp\left(-\frac{ n(f^*-f_l)^2 + 2\sigma (f^*-f_l)\sum_{j=1}^n Z^j }{2\sigma^2}\right) \} \subseteq \{\lim_{n\rightarrow \infty}p_l^n=0\}.
\end{align*}\normalsize

Since $\mathbb{P}(\Omega_{x,l})=1$, we have $\mathbb{P}\{\lim_{n\rightarrow \infty}p_l^n=0\} = 1$. \hfill$\Box$

%
%
%
%
%
%
%
%


\bigskip

Back to Lemma~\ref{p_is_0}, define $N^n(x,j)$ as the actual time when we make the $j$-th measurement on $x$ ($1\leq j\leq N^n(x)$). We have:
\small
\begin{align*}
p^n_l&= \frac{\prod_{j=1}^{n}\exp[-\frac{(\hat{y}^j-f(x^{j-1};\k_l))^2}{2\sigma^2}]}{\sum_{i=1}^L\prod_{j=1}^{n}\exp[-\frac{(\hat{y}^j-f(x^{j-1};\k_i)^2}{2\sigma^2}]}
\leq  \frac{\prod_{j=1}^{n}\exp[-\frac{(\hat{y}^j-f(x^{j-1};\k_l))^2}{2\sigma^2}]}{\prod_{j=1}^{n}\exp[-\frac{(\hat{y}^j-f(x^{j-1};\k^{^*})^2}{2\sigma^2}]}\\ &= \prod_{x\in\mathcal{X}}\frac{\prod_{j=1}^{N^n(x)}\exp[-\frac{(\hat{y}^{t^n(x,j)}-f(x^{j-1};\k_l))^2}{2\sigma^2}]}{\prod_{j=1}^{N^n(x)}\exp[-\frac{(\hat{y}^{N^n(x,j)}-f(x^{j-1};\k^{^*})^2}{2\sigma^2}]}\\ &\xlongequal{def}\prod_{x\in\mathcal{X}}r(x)\\
&= \prod_{x\in \{x: f_l(x)=f^*(x)\}}r(x) \prod_{\substack{x\in \{x:N^n(x)=\infty,\\f_l(x)\ne f(x;\k^*)\}}}r(x) \prod_{\substack{x\in \{x: N^n(x)\ne\infty,\\f_l(x)\ne f(x;\k^*)\} }}r(x).
\end{align*}
\normalsize

The first item equals $1$, and the third equals a constant number. For any $\omega\in \bigcap_{x\in\mathcal{X}} \Omega_{x,l}$, the second item has limit as $0$ by Claim 1. Since $\mathbb{P}(\bigcap_{x\in\mathcal{X}} \Omega_{x,l})=1-\mathbb{P}(\bigcup_{x\in\mathcal{X}}\Omega^c_{x,l})>1-\sum_{x\in\mathcal{X}}\mathbb{P}(\Omega_{x,l}^c)=1$, we have $p_l^n\xlongrightarrow{a.s}0$.


$\Box$

\bigskip\bigskip

\noindent\textbf{Lemma~\ref{f_non_neg}. }
For $\forall n\geq 0, \forall x\in \mathcal{X}$, the KGDP-$f$ score $\nu^{KGDP-f,n}(x)\geq 0$. Equality holds if and only if (1) there exists $x'$ such that $x'\in \mathop{\text{argmax}}_x f(x;\k_i)$ for all $i$ such that $p_i^n>0$, or (2) $p_i^n=0$ if $f(x;\k_i)\ne f(x;\k^{*})$.

\subsubsection*{Proof of Lemma~\ref{f_non_neg}}

Suppose at time $n$, the probabilities of our prior candidates are $(p_1^n, p_2^n, ..., p_L^n)$. Since $\max$ is a convex function, by Jensen's inequality,

\small
\begin{align*}
\nu^{KGDP,n}(x) &= \mathbb{E}^n[\max_{x'}\sum_{i=1}^Lf_i(x')p_i^{n+1}(x)|S^n=s,x^n=x]-\max_{x'}\sum_{i=1}^Lf_i(x')p_i^n\\
& \geq \max_{x'}\mathbb{E}[\sum_{i=1}^Lf_i(x')p_i^{n+1}(x)|S^n=s,x^n=x] - \max_{x'}\sum_{i=1}^Lf_i(x')p_i^n.
\end{align*}
\normalsize

We show that $\forall$ $x$, $x'$, $\mathbb{E}[\sum_{i=1}^Lf_i(x')p_i^{n+1}(x)|S^n=s,x^n=x] = \sum_{i=1}^Lf_i(x')p_i^n$. We know
\small
\begin{align*}
\mathbb{E}[\sum_{i=1}^Lf_i(x')p_i^{n+1}(x)|S^n=s,x^n=x] = \sum_{i=1}^Lf_i(x')\mathbb{E}[p_i^{n+1}(x)|S^n=s, x^n=x].
\end{align*}
\normalsize

Suppose the measurement at $x$ is $\hat{y}$, we have
\small
\begin{align*}
p_i^{n+1}(x)=\frac{p_i^n\exp[-\frac{(\hat{y}-f_i(x))^2}{2\sigma^2}]}{\sum_{j=1}^L p_j^n\exp[-\frac{(\hat{y}-f_j(x))^2}{2\sigma^2}]}.
\end{align*}
\normalsize

Based on our prior belief, the distribution of $\hat{y}$ is:
\small
\begin{align*}
\hat{y}\sim \left\{\begin{aligned}
& N(f_1(x;\k_1), \sigma^2), \text{with probability }p_1^n \\
& N(f_2(x;\k_2), \sigma^2), \text{with probability }p_2^n \\
& ...\\
& N(f_L(x;\k_L), \sigma^2), \text{with probability }p_L^n
\end{aligned} \right.
\end{align*}
\normalsize

Therefore,
\small
\begin{align*}
\mathbb{E} \left[p_i^{n+1}(x)\right] &= \mathbb{E} \left( \frac{p_i^n\exp[-\frac{(\hat{y}-f_i(x))^2}{2\sigma^2}]}{\sum_{j=1}^L p_j^n\exp[-\frac{(\hat{y}-f_j(x))^2}{2\sigma^2}]} \right)\\
&= \int_{-\infty}^{\infty} \sum_{k=1}^L p_k^n \frac{p_i^n\exp[-\frac{(\hat{y}-f_i(x))^2}{2\sigma^2}]}{\sum_{j=1}^L p_j^n\exp[-\frac{(\hat{y}-f_j(x))^2}{2\sigma^2}]}\frac{1}{\sqrt{2\pi}}\exp[-\frac{(\hat{y}-f_k(x))^2}{2\sigma^2}]d\hat{y}\\
&= \int_{-\infty}^{\infty} \frac{1}{\sqrt{2\pi}} p_i^n\exp[-\frac{(\hat{y}-f_i(x))^2}{2\sigma^2}] \frac{\sum_{k=1}^L p_k^n \exp[-\frac{(\hat{y}-f_k(x))^2}{2\sigma^2}]}{\sum_{j=1}^L p_j^n\exp[-\frac{(\hat{y}-f_j(x))^2}{2\sigma^2}]} d\hat{y}\\
&= p_i^n\int_{-\infty}^{\infty}\frac{1}{\sqrt{2\pi}} \exp[-\frac{(\hat{y}-f_i(x))^2}{2\sigma^2}] d\hat{y}\\
&= p_i^n.
\end{align*}\normalsize

Hence, for any $x$ and $x'$, $\mathbb{E}[\sum_{i=1}^Lf_i(x')p_i^{n+1}(x)|S^n=s,x^n=x] = \sum_{i=1}^Lf_i(x')p_i^n$. 

\vspace{2em}

For any convex function $f$, the general form of Jensen's inequality is $\mathbb{E}f(X)\geq f(\mathbb{E}X)$. The equality of Jensen's inequality requires that (1) $f$ is linear on $x$, or (2) $X$ is constant. (1) and (2) are equivalent to the two conditions stated in Lemma~\ref{f_non_neg} respectively:

For (1), $\max (X)$ is not a linear function in general, except in the aligned case as stated in Lemma~\ref{f_non_neg}, Condition 1.

For (2), this is equivalent to the following statement: $\forall$ $x'$, $\sum_{i=1}^L f(x';\k_i)p_i^{n+1}(x)$ is constant on all possible measurement outcomes $\hat{y}$ on $x$. This is further equivalent to the statement that the $p_i^{n+1}(x)$'s are constants. According to the updating rule given by {Equation~\ref{eq_p_update1}, this is equivalent to Condition 2 in Lemma~\ref{f_non_neg}.

$\Box$

\bigskip\bigskip

\noindent\textbf{Lemma~\ref{H_non_neg}. }
For $\forall n\geq 0, \forall x\in \mathcal{X}$, the KGDP-$H$ score $\nu^{KGDP-H,n}(x)\geq 0$. Equality holds if and only if $p_i^n=0$ if $f(x;\k_i)\ne f(x;\k^{*})$.

\subsubsection*{Proof of Lemma~\ref{H_non_neg}}

Suppose at time $n$, the probabilities of our prior candidates are $(p_1^n, p_2^n, ..., p_L^n)$. Since $x\log x$ is a convex function, by Jensen's inequality,

\small
\begin{align*}
\nu^{KGDP-H,n}(x) &= \mathbb{E}^n[\sum_{i=1}^L p_i^{n+1}(x)\log p_i^{n+1}(x) |S^n=s,x^n=x]-\sum_i^L p_i^n\log p_i^n\\
&= \sum_{i=1}^L \mathbb{E}^n[p_i^{n+1}(x)\log p_i^{n+1}(x) |S^n=s,x^n=x]-\sum_i^L p_i^n\log p_i^n\\
&\geq \sum_{i=1}^L \mathbb{E}^n\left[p_i^{n+1}(x)\right]\log \mathbb{E}^n\left[p_i^{n+1}(x)\right]-\sum_i^L p_i^n\log p_i^n\\
&= 0.
\end{align*}
\normalsize

Since $x\log x$ is nonlinear, the only case for equality is that $p_i^{n+1}(x)$ is a constant, which means all functions with positive probability have the same value at $x$.

$\Box$

\bigskip\bigskip

\noindent\textbf{Lemma~\ref{v_score0}. }
For $\forall \omega\in\Omega_0$, $\forall x\in\mathcal{X}'(\omega)$, we have $\lim_{n\rightarrow\infty}\nu^{KGDP-f,n}(x)(\omega)=0$, and\\ $\lim_{n\rightarrow\infty}\nu^{KGDP-H,n}(x)(\omega)=0$.

\subsubsection*{Proof of Lemma~\ref{v_score0}}

First, note that for fixed $x$ and $n$, $\nu^{KGDP-f,n}(x)$ and $\nu^{KGDP-H,n}(x)$ are functions of $\vec{p}=(p_1^n,...,p_L^n)$. In order to concentrate on $\vec{p}$, we pick any fixed $\omega\in\Omega_0$, any fixed $x\in\mathcal{X}'(\omega)$, and then denote $\nu^{KGDP-f,n}(x)$ and $\nu^{KGDP-H,n}(x)$ as $\nu_f(\vec{p})$ and $\nu_H(\vec{p})$. For this fixed $x$, we assume there are $h$ $(1\leq h\leq L)$ functions equal to $f(x;\k^*)$ at $x$. Without loss of generality, we assume they are $\k_1,...,\k_h$. That is, $f(x;\k_i)=f(x;\k^*)$, for $\forall i\in [h]$.

Let $\vec{p_0} = (0,0,...,0, p_{h+1}, ..., p_L)$, where $p_{h+1}+...+p_L=1$. By Lemma~\ref{f_non_neg} and Lemma~\ref{H_non_neg}, $\vec{p_0}$ meets the conditions for $\nu_f(\vec{p})$ and $\nu_H(\vec{p})$ to be $0$. That is, $\nu_f(\vec{p_0})=0$, $\nu_H(\vec{p_0})=0$.

Obviously, $\nu_f(\vec{p})$ and $\nu_H(\vec{p})$ are continuous in $\vec{p}$. We take $\nu_f(\vec{p})$ as an example.
\begin{fleqn}[0pt]
\begin{flalign*}
&{\small \sum_{i=1}^L f_i(x')p_i^n\exp\left[-\frac{(\hat{y} - f_i(x))^2}{2\sigma^2}\right]} \text{ is continuous. Therefore, we know that }\\
&{\small \max_{x'}\left(\sum_{i=1}^L f_i(x')p_i^n\exp\left[-\frac{(\hat{y} - f_i(x))^2}{2\sigma^2}\right]\right)} \text{ is also continuous.} \text{ Hence, }\\
&\int_{-\infty}^{+\infty}\max_{x'}\left(\sum_{i=1}^L f_i(x')p_i^n\exp\left[-\frac{(\hat{y} - f_i(x))^2}{2\sigma^2}\right]\right)d\hat{y} \text{ is continuous, too. Then by }\\&\text{Equation~\ref{eq_KGDP2}, } \nu_f(\vec{p}) \text{ is continuous}.
\end{flalign*}
\end{fleqn}


Since $\vec{p}$ is defined on a compact set, by the Heine-Cantor theorem, both $\nu_f(\vec{p})$ and $\nu_H(\vec{p})$ are uniformly continuous. That is, for $\forall \e>0$, there exists $\delta>0$, such that $\forall \vec{p_1}\ne \vec{p2}$, $|\vec{p_1}-\vec{p_2}|<\delta$, $|\nu_f(\vec{p_1})-\nu_f(\vec{p_2})|<\e$, $|\nu_H(\vec{p_1})-\nu_H(\vec{p_2})|<\e$.

For this $\omega\in\Omega_0$ and $x\in\mathcal{X}'(\omega)$, since $\lim_{n\rightarrow\infty}p_i^n(\omega)=0$, $1\leq i\leq h$, there exists $N$ such that for $\forall n>N$, $p_i^n(\omega)<\delta/(L+1)$. Let $\vec{p'}_n=(0,...,0, p^n_{h+1},...,p^n_{L-1}, p^n_L+\sum_{i=1}^h p_i^n)$. Then $v_f(\vec{p'})=v_H(\vec{p'})=0$, and $|\vec{p}_n - \vec{p'}_n|<\sqrt{h+h^2}\cdot \frac{\delta}{L+1}<\delta$.

Therefore, $\nu_f(\vec{p_n})<\e$, $\nu_H(\vec{p_n})<\e$, $\forall n>N$. Hence, for $\forall \omega\in\Omega_0$, and $x\in\mathcal{X}'(\omega)$, $\lim_{n\rightarrow\infty}\nu^{KGDP-f,n}(x)=0$, $\lim_{n\rightarrow\infty}\nu^{KGDP-H,n}(x)=0$.

$\Box$

\bigskip\bigskip

\noindent\textbf{Lemma~\ref{measure_suf1}. }
For any $\omega\in\Omega_0$, the alternatives measured infinitely often under the KGDP-$f$ or KGDP-$H$ policy constitute a sufficient set.

\subsubsection*{Proof of Lemma~\ref{measure_suf1}}

Assume the contrary. For any $\omega\in\Omega_0$, if $\mathcal{X}_{\infty}(\omega)$ is not a sufficient set, then there exist other $\k$'s that fit the true values of $\mathcal{X}_{\infty}(\omega)$. Assume there are $h$ such $\k$'s, $1\leq h\leq L$. WLOG, assume they are $\k_1$,...,$\k_h$. So for any $l\in[h]$,
\begin{align*}
f(x;\k_l) = f(x;\k^*), \text{ for }\forall x\in\mathcal{X}_\infty(\omega).
\end{align*}

By Lemma~\ref{p_is_0}, for any $h+1\leq l \leq L$, $\lim_{n\rightarrow\infty}p_l^n(\omega)=0$.

We now show that for any $1\leq l \leq h$, $\lim_{n\rightarrow\infty}p_l^n(\omega)>0$ exists.

For this $\omega$, let $\mathcal{L}^n_1,\mathcal{L}^n_2,...,\mathcal{L}^n_L$ be the likelihood of the $L$ $\k$'s according to the first $n$ measurements. Hence,\small
\begin{align*}
\mathcal{L}^n_l = \prod_{i=1}^n \exp(-\frac{(\hat{y}^i_x - f(x;\k_l))^2}{2\sigma^2}).
\end{align*}\normalsize

Let $T$ be the last time that we measure any $x$ out of $\mathcal{X}_\infty(\omega)$. After time $T$, for any $l\in\{1,2,...,h\}$,
{\small
\begin{align*}
&p_l^n(\omega) = \frac{\mathcal{L}^n_l}{\sum_{m=1}^h \mathcal{L}^n_m + \sum_{m=1}^h \mathcal{L}^n_m }\\
=& \frac{\mathcal{L}^T_l \prod_{i=T+1}^n \exp(-\frac{(\hat{y}^i_x - f(x;\k^{*}))^2}{2\sigma^2})}{\sum_{m=1}^h \mathcal{L}^T_m \prod_{i=T+1}^n \exp(-\frac{(\hat{y}^i_x - f(x;\k^{*}))^2}{2\sigma^2}) + \sum_{m=h+1}^L \mathcal{L}^T_m \prod_{i=T+1}^n \exp(-\frac{(\hat{y}^i_x - f(x;\k_{m}))^2}{2\sigma^2})}\\ 
=& \frac{\mathcal{L}^T_l}{\sum_{m=1}^h \mathcal{L}^T_m + \sum_{m=h+1}^L \mathcal{L}^T_m \prod_{i=T+1}^n \frac{\exp(-\frac{(\hat{y}^i_x - f(x;\k_{m}))^2}{2\sigma^2})}{\exp(-\frac{(\hat{y}^i_x - f(x;\k^*))^2}{2\sigma^2})}}.
\end{align*}
}

In Proof of Lemma~\ref{p_is_0}, we have shown that $\lim_{n\rightarrow\infty} \prod_{i=T+1}^n \frac{\exp(-\frac{(\hat{y}^i_x - f(x;\k_{m}))^2}{2\sigma^2})}{\exp(-\frac{(\hat{y}^i_x - f(x;\k^*))^2}{2\sigma^2})} = 0$. Therefore,
\small\begin{align*}
\lim_{n\rightarrow\infty}p_l^n(\omega)=\frac{\mathcal{L}^T_l}{\sum_{m=1}^h \mathcal{L}^T_m}>0 \text{ exists}.
\end{align*}\normalsize

By Lemma~\ref{v_score0}, for $\forall x\in\mathcal{X}_\infty(\omega)$, $\lim_{n\rightarrow\infty}\nu^{KGDP-f,n}(x)(\omega)=0$, \\ $\lim_{n\rightarrow\infty}\nu^{KGDP-H,n}(x)(\omega)=0$. For $\forall x\in\mathcal{X}_\infty(\omega)^c$, $\lim_{n\rightarrow\infty}\nu^{KGDP-f,n}(x)(\omega)>0$, $\lim_{n\rightarrow\infty}\nu^{KGDP-H,n}(x)(\omega)>0$. Since the KGDP-$f$ (or KGDP-$H$) policy  always chooses the $x$ with the largest KGDP-$f$ (or KGDP-$H$) score, it will measure some $x\in\mathcal{X}'(\omega)^c$ after $T$, which is contradictory to our assumption.

$\Box$

\bigskip\bigskip
\noindent\textbf{Theorem~\ref{KGDP-f_con}. }
Non-resampling KGDP-$f$ with truth from prior is asymptotically optimal in finding both the optimal alternative and the correct parameter. The same holds for KGDP-$H$.

\subsubsection*{Proof of Theorem~\ref{KGDP-f_con}}

Under either KGDP-$f$ or KGDP-$H$ policy, for any $\k_l\ne \k^*$, Lemma~\ref{measure_suf1} implies that for any $\omega\in\Omega_0$, there exists $x\in\mathcal{X}_\infty(\omega)$ such that $f(x;\k_l)\ne f(x;\k^*)$. By Lemma~\ref{p_is_0}, we have $\mathbb{P}(\lim_{n\rightarrow\infty}p_l^n=0)=1$. That is, $\mathbb{P}(\lim_{n\rightarrow\infty}p^n(\k^*)=1)=1$, and $\mathbb{P}(\bar{f}(x)=f(x;\k^*))=1$.

$\Box$

\bigskip\bigskip
\noindent\textbf{Lemma~\ref{measure_suf}. }
For any $\omega\in\Omega_1$, the alternatives measured infinitely often under the KGDP-$f$ or KGDP-$H$ policy constitute a sufficient set. We denote this set as $\mathcal{X}_\infty(\omega)$.

\subsubsection*{Proof of Lemma~\ref{measure_suf}}


For any $\omega\in\Omega_1$, let $T$ be the last time that we measure any $x$ out of $\mathcal{X}_\infty(\omega)$.

Assume the contrary (thus $T>0$). Since $\mathcal{X}_\infty(\omega)$ is not a sufficient set, there exists at least one $\k\in\mathbb{K}$, $\k\ne\k^*$ such that $f(x;\k)=f(x;\k^*)$ for $\forall x\in\mathcal{X}_\infty(\omega)$. Denote the set of these $\k$'s as $\mathbb{K}'$. 

Let $\mathcal{L}^n_1,\mathcal{L}^n_2,...,\mathcal{L}^n_K$ be the likelihood of the $K$ $\k$'s according to the first $n$ measurements. If we regard $\mathbb{K}$ as the ``candidate set'' and let $w^n(\k)$ denote its ``probability'', then as Proof of Lemma~\ref{measure_suf1} shows, for $\forall \k\in\mathbb{K}'\bigcup\{\k^*\}$, $\lim_{n\rightarrow\infty} w^n(\k)=\frac{\mathcal{L}^T_{\k}}{\sum_{i=1}^K \mathcal{L}^T_i}>0$. For all other $\k$'s, by Lemma~\ref{p_is_0}, $\lim_{n\rightarrow\infty}w^n(\k)=0$. $w^n(\k)$ is actually the weight of $\k$ when we do resampling. Hence, as $n$ gets larger, $\k$'s in $\mathbb{K}'\bigcup\{\k^*\}$ always rank higher than others to enter the small pool (i.e, the sub-level set), from which we resample.

\bigskip

We refer to the event as $A_n$ that at least two $\k$'s in $\mathbb{K}'\bigcup \{\k^*\}$ are in the candidate set at time $n$. We first show that $A_n$ happens infinitely often. That is, $\mathbb{P}_\k(\bigcap_{m\geq 1}\bigcup_{n\geq m}A_n)=1$, where $\mathbb{P}_\k$ indicates the fact that this probability is calculated in the $(\Omega_\k, \mathcal{F}_\k, \mathbb{P}_\k)$ space.



There are two cases at a resampling step: (1) $|\mathbb{L}_{rm}|\geq 1$ or (2) $|\mathbb{L}_{rm}|=1$ (remember that $L_{rm}$ is the set of candidates to remove). Since resampling happens infinite often, at least either (1) or (2) happens infinitely often.


We first assume (1) happens infinitely often, indexed by a subsequence of times $\{s_n\}$. Let $B_{s_n}$ be the event that at time $s_n$, two $\k$'s from $\mathbb{K}'\bigcup\{\k^*\}$ are selected in the first draw. (Remember that we may have multiple draws if resampling keeps being triggered. However, for $n$ large enough, if two such $\k$'s are selected in the first draw, they will not be dropped in this iteration. Hence, $B_{s_n}\subset A_{s_n}$) Since all $w^n(\k)$'s have limits, we know $\lim_{n\rightarrow\infty} \mathbb{P}_\k (B_{s_n}) >0$. Therefore, $\sum_{n=1}^\infty \mathbb{P}_\k(B_{s_n})=\infty$. On the other hand, for a fixed $\omega$, $w^n(\k)$'s are all decided, and therefore $B_{s_n}$ are pairwise independent. By the Borel-Cantelli Lemma, the probability that $B_{s_n}$ happens infinitely often, i.e, $\mathbb{P}_\k(\bigcap_{m\geq 1}\bigcup_{s_n\geq m}B_{s_n})=1$. We know that $B_{s_n}$ is a subset of $A_{s_n}$. Hence, $\mathbb{P}_\k(\bigcap_{m\geq 1}\bigcup_{s_n\geq m}A_{s_n})=1$, and furthermore, $\mathbb{P}_\k(\bigcap_{m\geq 1}\bigcup_{n\geq m}A_{n})=1$.

Otherwise, if (1) happens a finite number of times and (2) happens infinitely often indexed by $\{t_n\}$, we assume the contrary. That is, there exists time $T_2$, such that after $T_2$, among the $L$ candidates there is at most one $\k\in \mathbb{K}'\bigcup \{\k^*\}$. At time $t_n$, since $|\mathbb{L}_{rm}|=1$, by assumption, the other $L-1$ candidates must be identical, and the newly drawn one must be the same as the other $L-1$ ones. The probability that this keeps happening from time $t_M$ to infinity, is $\lim_{N\rightarrow\infty} \prod_{n=M}^N w^{t_n}(\k)=0$. Therefore, our assumption does not hold and $A_n$ stills happens infinitely often!


\bigskip

Hence, for any fixed $\omega\in\Omega_1$, with probability $1$, we have the subsequence $\{q_n(\omega)\}$, such that at these times, the candidates contain at least two $\k$'s in $\mathbb{K}'\bigcup \{\k^*\}$, denoted as $\k_1^n$ and $\k_2^n$. For any $x\in \mathcal{X}_\infty(\omega)$, $\lim_{n\rightarrow\infty} \nu^{KGDP-f,n}(x)(\omega)=0$,\\ $\lim_{n\rightarrow\infty} \nu^{KGDP-H,n}(x)(\omega)=0$; while for any $x$ such that $f(x;\k_1^n)\ne f(x;\k_2^n)$ and any $q_n$, $\nu^{KGDP-f,q_n}(x)(\omega)>0$, $\nu^{KGDP-H,q_n}(x)(\omega)>0$. As there are only finite combinations of $\k_1$ and $\k_2$ and each $\k\in\mathbb{K}'\bigcap \{\k^*\}$ has a positive probability in the limit, there exists $\e>0$, such that $\sup_{} \nu^{KGDP-f,q_n}(x)(\omega)>\e$, $\nu^{KGDP-H,q_n}(x)(\omega)>\e$. Since $q_n$ can be larger than $T$, this is contradictory to the fact that KGDP-$f$ or KGDP-$H$ will not measure  any $x$ out of $\mathcal{X}_\infty(\omega)$ after time $T$.

$\Box$

\bigskip\bigskip
\noindent\textbf{Theorem~\ref{re_KGDP-f_con}. }
Resampling KGDP-$f$ is asymptotically optimal in finding the optimal alternative and the correct parameter. The same holds for KGDP-$H$.

\subsubsection*{Proof of Theorem~\ref{re_KGDP-f_con}}

For any fixed $\omega\in\Omega_1$, let $C_n$ ($C_n\in\mathcal{F}_\k$) be the set of events that the $L$ candidates includes $\k^*$ at time $n$.

By Lemma~\ref{measure_suf}, $\k^*$ is the only parameter in $\mathbb{K}$ that fits $\mathcal{X}_\infty(\omega)$. Assume $\{s_n\}$ is the subsequence of times when resampling happens. Then $\mathbb{P}_\k (C_{s_n})>\frac{w^{s_n}(\k^*)(\omega)}{\sum_{i=1}^K w^{s_n}_i(\omega)}$, where the right hand side is the probability that $\k^*$ is chosen in the first draw.

We show $\mathbb{P}_\k(\bigcup_{m\geq 1}\bigcap_{n\geq m}C_n)=1$, which means that $\k^*$ is in the candidate set for all but a finite number of times.

Since resampling only happens on $\{s_n\}$, we know $\mathbb{P}_\k(\bigcup_{m\geq 1}\bigcap_{n\geq m}C_n)=$ \\$\mathbb{P}_\k(\bigcup_{m\geq 1}\bigcap_{n\geq m}C_{s_n})$. Intuitively, this means if after time $T$, every resampling step chooses $\k^*$ for the candidate set, then $\k^*$ appears in the candidate set every iteration after $T$. 

For this fixed $\omega$, we know\small
\begin{align*}
\mathbb{P}_\k (C_{s_n}^c) = 1- \mathbb{P}_\k (C_{s_n})< \frac{\sum_{\k\in\mathbb{K}, \k\ne\k^*} w^{s_n}(\k)}{\sum_{\k\in\mathbb{K}} w^{s_n}(\k)}<\frac{\sum_{\k\in\mathbb{K}, \k\ne\k^*} w^{s_n}(\k)}{w^{s_n}(\k^*)}.
\end{align*}\normalsize

From proof of Lemma~\ref{p_is_0} (for detail, check out Equation~(\ref{Eq_bound1}) and (\ref{Eq_bound2})), for this $\omega\in\Omega_1$, we know there exists $N_1$, such that for all $n>N_1$ and any $\k\ne\k^*$,\small
\begin{align*}
\frac{w^{s_n}(\k)}{w^{s_n}(\k^*)} < \exp(-n\frac{(f^*-f_l)^2}{4\sigma^2}).
\end{align*}\normalsize

Assume $\underset{l,x}{\sup} \frac{(f^*-f_l)^2}{4\sigma^2}=\rho$. Hence, for $n>N_1$, $\mathbb{P}_\k (C_{s_n}^c)<K e^{-n\rho}$. Therefore,
\vspace{-0.5em}\small\begin{align*}
\sum_{n\geq 1}\mathbb{P}_\k(C_{s_n}^c) < \sum_{n \geq 1}\mathbb{P}_\k(D_{n}^c) < N_1 + \sum_{n \geq N_1+1}Ke^{-n\rho} = N_1+ K \frac{e^{-(N_1+1)\rho}}{(1-e^{-\rho})}<\infty.
\end{align*}
\normalsize
\vspace{-0.5em}

By the Borel-Cantelli Lemma, $\mathbb{P}_\k(\bigcap_{m\geq 1}\bigcup_{n\geq m}C_{s_n}^c)=0$. Hence,\\ $\mathbb{P}_\k(\bigcup_{m\geq 1}\bigcap_{n\geq m}C_{n})=\mathbb{P}_\k(\bigcup_{m\geq 1}\bigcap_{n\geq m}C_{s_n})=1$.

Hence, for any $\omega\in\Omega_1$, with probability $1$, there exists $T(\omega)$, such that for any $n\geq T(\omega)$, $\k^*$ is in the candidate set, and $\lim_{n\rightarrow\infty}p^n(\k^*)(\omega)=1$.


This holds for every $\omega\in\Omega_1$, hence, $\mathbb{P}_{\Omega\times\Omega_\k} \{ \lim_{n\rightarrow\infty}p^{n}(\k^*) = 1 \} = 1$, where $\mathbb{P}_{\Omega\times\Omega_\k}$ indicates the probability in the full probability space $(\Omega_f, \mathcal{F}_f, \mathbb{P}_f)$, which is the product space of $(\Omega, \mathcal{F}, \mathbb{P})$ and $(\Omega_\k, \mathcal{F}_\k, \mathbb{P}_\k)$. Since we find $\k^*$ with probability $1$, we also find the optimal $x$ with probability $1$.

$\Box$

\bibliographystyle{siamplain}
\bibliography{library,reference}
\end{document}